%% file: main.tex
\documentclass{article}

\usepackage{microtype}
\usepackage{graphicx}
\usepackage{subcaption}
\usepackage{booktabs} 

\usepackage{hyperref}

\usepackage[preprint]{icml2026}

\usepackage{amsmath}
\usepackage{amssymb}
\usepackage{mathtools}
\usepackage{amsthm}

\usepackage{hyperref}
\usepackage{url}

\usepackage{microtype}
\usepackage{graphicx}
\usepackage{subcaption}
\usepackage{booktabs} 

\usepackage{amsmath}	
\usepackage{amssymb}	
\usepackage{booktabs}
\usepackage{times}
\usepackage{caption}
\usepackage{microtype}
\usepackage{epsfig}
\usepackage[table,xcdraw,dvipsnames]{xcolor}
\usepackage{caption}
\usepackage{float}
\usepackage{placeins}
\usepackage{color, colortbl}
\usepackage{stfloats}
\usepackage{enumitem}
\usepackage{tabularx}
\usepackage{xstring}
\usepackage{multirow}
\usepackage{xspace}
\usepackage{url}
\usepackage{subcaption}
\usepackage{makecell}
\usepackage[hang,flushmargin]{footmisc}

\newif\ifreview 
\newif\ifarxiv 
\newif\ifcamera 
\newif\ifrebuttal 

\ifcamera \usepackage[accsupp]{axessibility} \fi

\usepackage{amsthm}

\usepackage{subcaption}   
\usepackage{graphicx}
\usepackage{caption}
% ...

\ifarxiv  \fi

\usepackage{amsmath}
\usepackage{amssymb}
\usepackage{mathtools}
\usepackage{amsthm}
\usepackage{mdframed}
\usepackage{tikz}
\usepackage{etoolbox}
\usepackage[capitalize,noabbrev]{cleveref}
\usepackage{thmtools}

\theoremstyle{plain}
\newtheorem*{theorem*}{Theorem}
\newtheorem{theorem}{Theorem}[section]
\newtheorem{proposition}[theorem]{Proposition}
\newtheorem{lemma}[theorem]{Lemma}

\newtheorem{corollary}[theorem]{Corollary}
\theoremstyle{definition}
\newtheorem{definition}[theorem]{Definition}

\theoremstyle{remark}
\newtheorem{remark}[theorem]{Remark}

\usepackage{booktabs}
\definecolor{OurHL}{RGB}{242, 242, 242} 

\definecolor{OurHL}{RGB}{242,242,242}   
\definecolor{OurCrim}{HTML}{CC785C}     

\newcommand{\best}[1]{\textcolor{OurCrim}{\textbf{#1}}}

\newcommand{\secondbest}[1]{\underline{#1}}

\usepackage{xcolor}
\usepackage{empheq}

\usepackage{booktabs}
\usepackage{siunitx}
\usepackage{wrapfig}

\usepackage{mdframed}  

\surroundwithmdframed[
  linewidth=0.5pt,           
  linecolor=gray!20,         
  tikzsetting={dash pattern=on 3pt off 3pt}, 
  backgroundcolor=gray!10,    
  roundcorner=4pt,           
  innerleftmargin=6pt,       
  innerrightmargin=6pt,
  innertopmargin=8pt,
  innerbottommargin=6pt
]{theorem}

\surroundwithmdframed[
  linewidth=0.5pt,
  linecolor=gray!20,
  tikzsetting={dash pattern=on 3pt off 3pt},
  backgroundcolor=gray!10,
  roundcorner=4pt,
  innerleftmargin=6pt,
  innerrightmargin=6pt,
  innertopmargin=8pt,
  innerbottommargin=6pt
]{proposition}

\surroundwithmdframed[
  linewidth=0.5pt,
  linecolor=gray!20,
  tikzsetting={dash pattern=on 3pt off 3pt},
  backgroundcolor=gray!10,
  roundcorner=4pt,
  innerleftmargin=6pt,
  innerrightmargin=6pt,
  innertopmargin=8pt,
  innerbottommargin=6pt
]{criteria}

\surroundwithmdframed[
  linewidth=0.5pt,
  linecolor=gray!20,
  tikzsetting={dash pattern=on 3pt off 3pt},
  backgroundcolor=gray!10,
  roundcorner=4pt,
  innerleftmargin=6pt,
  innerrightmargin=6pt,
  innertopmargin=8pt,
  innerbottommargin=6pt
]{theoremx}

\usepackage[capitalize,noabbrev]{cleveref}

\theoremstyle{plain}

\usepackage[textsize=tiny]{todonotes}

\icmltitlerunning{ZEUS: Accelerating Diffusion Models with Only Second-Order Predictor}

\begin{document}

\twocolumn[
  \icmltitle{ZEUS: Accelerating Diffusion Models with Only Second-Order Predictor}



  \icmlsetsymbol{equal}{*}

\begin{icmlauthorlist}
    \icmlauthor{Yixiao Wang}{equal,CS,stats}
    \icmlauthor{Ting Jiang}{equal,duke}
    \icmlauthor{Zishan Shao}{equal,duke,stats}
    \icmlauthor{Hancheng Ye}{duke}
    \icmlauthor{Jingwei Sun}{duke}
    \icmlauthor{Mingyuan Ma}{duke}
    \icmlauthor{Jianyi Zhang}{duke}
    \icmlauthor{Yiran Chen}{duke}
    \icmlauthor{Hai Li}{duke}
  \end{icmlauthorlist}

  \icmlaffiliation{duke}{Department of Electrical and Computer Engineering, Duke University, Durham, NC, USA}
  
  \icmlaffiliation{CS}{Department of Computer Science, Duke University, Durham, NC, USA}

 \icmlaffiliation{stats}{Department of Statistical Science, Duke University, Durham, NC, USA}

  \icmlcorrespondingauthor{Yixiao Wang}{yixiao.wang@duke.edu}

  \icmlkeywords{Machine Learning, ICML}

  \vskip 0.2in
] 

\printAffiliationsAndNotice{\icmlEqualContribution}  

\begin{abstract}
Denoising generative models deliver high-fidelity generation but remain bottlenecked by inference latency due to the many iterative denoiser calls required during sampling. 
Training-free acceleration methods reduce latency by either sparsifying the model architecture or shortening the sampling trajectory.
Current training-free acceleration methods are more complex than necessary: higher-order predictors amplify error under aggressive speedups, and architectural modifications hinder deployment.
Beyond $2\times$ acceleration, step skipping creates structural scarcity—at most one fresh evaluation per local window—leaving the computed output and its backward difference as the only causally grounded information.
Based on this, we propose \textbf{ZEUS}, an acceleration method that predicts reduced denoiser evaluations using a second-order predictor, and stabilizes aggressive consecutive skipping with an interleaved scheme that avoids back-to-back extrapolations.
ZEUS adds essentially zero overhead, no feature caches, and no architectural modifications, and it is compatible with different backbones, prediction objectives, and solver choices.
Across image and video generation, ZEUS consistently improves the speed--fidelity performance over recent training-free baselines, achieving up to $3.2\times$ end-to-end speedup while maintaining perceptual quality. Our code is available at: \url{https://github.com/Ting-Justin-Jiang/ZEUS}.
\end{abstract}

\input{texts/introduction}

\input{texts/relatedwork}

\input{texts/methodology}

\input{texts/experiment}

\input{texts/conclusions}

\input{texts/impact_statement}

\nocite{langley00}

\bibliography{refs}
\bibliographystyle{icml2026}

\newpage
\appendix
\onecolumn

\input{texts/appendix}

\input{texts/appendix_experiments}




\end{document}

%% file: texts/introduction.tex
\section{Introduction}
\label{sec:intro}
Denoising generative models define a forward process that progressively corrupts data into noise, and train a network to predict the corresponding denoising direction at each noise level.
Sampling then follows the learned reverse-time dynamics, which can be implemented by integrating the probability-flow ODE (PF-ODE)~\citep{song2020score}.
Numerical solvers~\citep{lu2022dpm,karras2022elucidating} discretize this ODE into numbers of integration steps, each requiring a full evaluation through a large denoiser network. This iterative evaluation creates substantial latency as industry models are becoming increasingly large: generating a single 5-second video clip with Wan2.1-14B takes over 20 minutes on an NVIDIA H100 GPU. Training-free acceleration methods address this bottleneck by skipping or sparsifying a subset of denoiser evaluations, reducing wall-clock time without modifying off-the-shelf weights. 

Existing training-free methods exploit two primary forms of redundancy: \emph{architectural redundancy}, where attention computations are sparsified within each denoiser evaluation~\citep{bolya2023token,yuan2024ditfastattn}, and \emph{trajectory redundancy}, where a numerical predictor leverages outputs or intermediate features from the current step to approximate future evaluations~\citep{ma2024deepcache,liu2025taylor}.
However, aggressive end-to-end acceleration (i.e., beyond $2\times$ speedup) with these methods often degrades sample quality substantially, producing visible artifacts or outputs that diverge from the unaccelerated baseline. 
In this paper, we state a counterintuitive position: \textbf{training-free acceleration methods are more complex than necessary.}

\textbf{(i) Current training-free acceleration has a higher numerical order than necessary.}
Recent trajectory-based methods increase the asymptotic order of their predictors, claiming that higher-order approximation yields greater accuracy.
TaylorSeer, for instance, employs Taylor expansions of order $k \geq 3$ to approximate skipped steps~\citep{liu2025taylor}.
However, this claim breaks down under aggressive acceleration.
When the target speedup exceeds $2\times$, at most one fresh denoiser evaluation can occur within any three-step window.
Under this scarcity, higher-order predictors must either chain approximations derived from earlier (already approximate) outputs, or incorporate stale evaluations from distant steps.
Both choices accumulate error rather than reduce it.
We show that the executed trajectory at any full evaluation step can be characterized by exactly two quantities: the computed output and its backward difference from the previous output.
We call this pair the \best{observed information set}—the minimal sufficient information for prediction under scarcity.
The practical implication is striking: under identical end-to-end speedup, \emph{second-order prediction outperforms higher-order methods} (Figure~\ref{fig:zeus-ablations}).
More is not better; additional predictor complexity degrades fidelity rather than improving it.

\begin{figure*}
    \centering
    \includegraphics[width=0.93\textwidth]{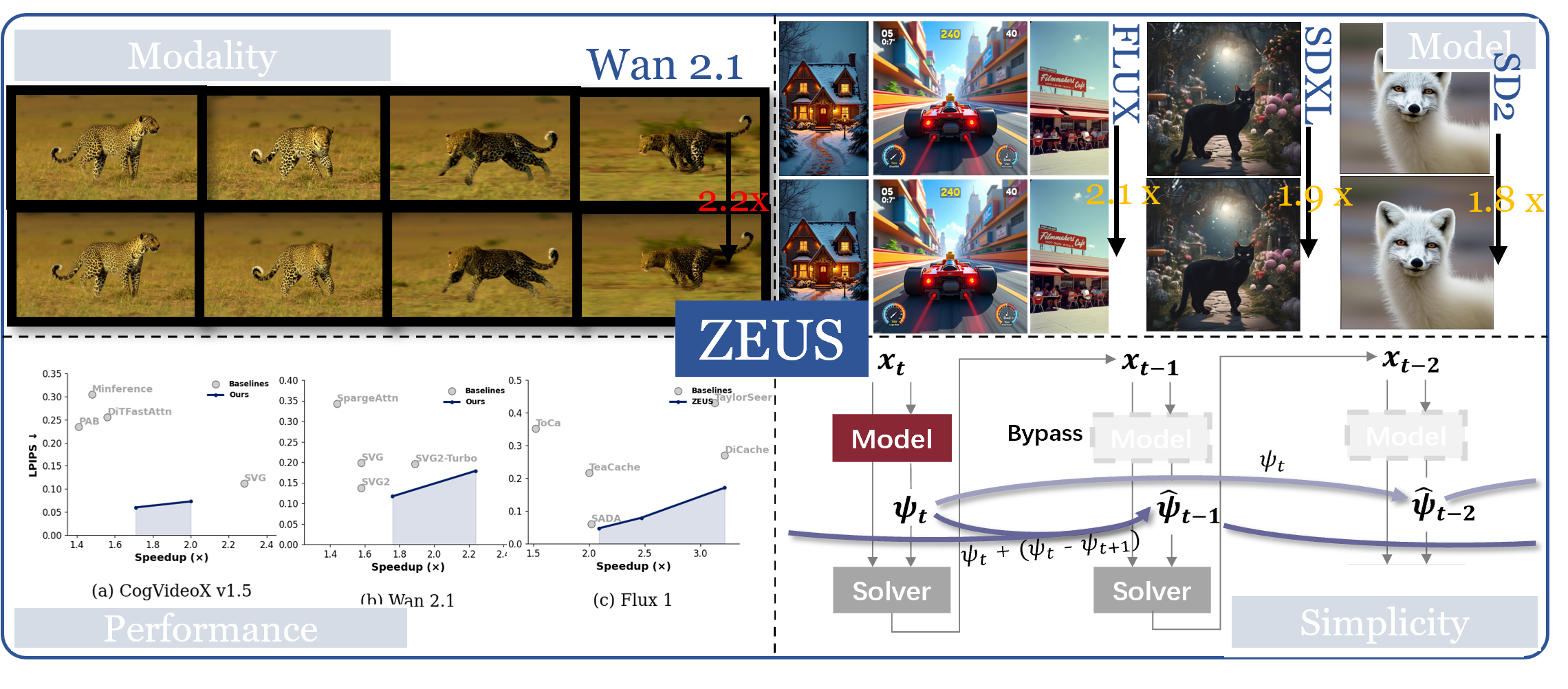}
     \caption{\textbf{Overview of ZEUS.}
ZEUS is a training-free acceleration framework for ODE-based generative models with four key properties:
(i)~modality-agnostic—applicable to image, video generation;
(ii)~parameterization-agnostic—compatible with $\epsilon$-, $v$-, and flow-prediction objectives;
(iii)~state-of-the-art speed–fidelity tradeoff—outperforming more complex methods with minimal overhead;
(iv)~minimal integration effort—fewer than 20 lines of code.
See Section 4 for details.}
  \label{fig:teaser}
\end{figure*}



\textbf{(ii) Training-free acceleration is harder to deploy than necessary.}
Production models evolve rapidly, varying in backbone architecture~\citep{ronneberger2015u, peebles2023scalable, esser2024scaling} and prediction objective~\citep{salimansprogressive, liu2022flow}.
A practical acceleration method must generalize across this diversity with minimal per-model engineering.
Yet existing approaches impose substantial deployment burdens.
Architectural methods design sparse attention patterns that require custom CUDA kernels and operator fusion: modifications that must be re-engineered for each new backbone.
Trajectory-based methods cache intermediate features, typically one tensor per attention layer, incurring memory overhead that scales with model depth.
On large video models such as Wan2.1-14B, this overhead alone can exceed available GPU memory.
By contrast, our approach requires only the previous denoiser output, a single tensor already collected by the numerical solver.
The predictor involves only elementwise operations with $\mathcal{O}(1)$ time complexity and no architectural intrusion.

In this paper, we propose Zero-cost Extrapolation-based Unified Sparsity (ZEUS).
ZEUS uses a second-order numerical predictor that extrapolates the next denoiser output from the \emph{observed information set} at the most recent full-evaluation step.
However, under end-to-end speedups beyond $2\times$, na\"ively chaining extrapolations across consecutive reduced steps can amplify error and cause instability.
ZEUS resolves this by \emph{reusing} the same observed information set in an interleaved scheme, which preserves the numerical precision and retain stable for long-range approximation. 


We provide comprehensive theoretical and empirical evidence for our framework.
ZEUS consistently improves the speed--quality trade-off on recent images (SD-2/SDXL/FLUX) and videos (Wan~2.1/CogVideoX) generation models, reaching up to $3.22\times$ and $2.24\times$ speedup while preserving sample fidelity (Visualized in Figure \ref{fig:teaser}).
It runs with unmodified U-Net, DiT, and MMDiT across five prediction objectives (Table~\ref{tab:parameterization_main}), requiring fewer than 20 lines of code to integrate into any Huggingface diffusers pipeline.

To summarize, our contributions are threefold:
\begin{enumerate}[itemsep=0pt, topsep=0pt]
    \item We reveal that current training-free accelerators are often unnecessarily complicated: a local second-order predictor suffices to approximate skipped steps.
    \item We propose an interleaved caching schedule that remains stable as the run length of consecutive reduced steps increases. 
    \item Extensive experiments and ablations confirm second-order prediction and the interleaved scheme, yielding strong speed--fidelity trade-offs across diverse settings.
\end{enumerate}

%% file: texts/relatedwork.tex
\section{Related Work}
\label{sec:related}

\paragraph{Denoising Generative Models.}
Denoising generative models~\cite{sohl2015deep, ho2020denoising, song2020score, nichol2021improved}
define a forward process that gradually corrupts data into Gaussian noise, and train a network to predict how to undo this corruption; sampling then starts from noise and runs the learned reverse-time dynamics back to data. Nowadays, diffusion models~\citep{ho2020denoising, nichol2021improved} and flow-based models~\citep{liu2022flow, albergo2022building, lipmanflow} have become mainstream frameworks for scalable generative modeling, with recent transformer-based backbones~\citep{peebles2023scalable} achieving high-fidelity generation across image, video, text, and audio modalities.

Sampling in diffusion/flow-based models is often implemented by integrating the \emph{probability-flow ODE}~\citep{song2019generative, song2020denoising},
with different but largely equivalent parameterizations and training objectives ($\boldsymbol{\epsilon}$, $\mathbf{x}_0$, $v$, the score, or a flow-matching velocity)~\citep{song2020score,song2019generative,song2020denoising,
song2023consistency,kim2023consistency,lipmanflow}. Building on the ODE formulation, higher-order solvers (samplers)~\citep{lu2022dpm, lu2022dpmpp, karras2022elucidating} can use larger step sizes while staying accurate, thus needing fewer steps to sample. In this work, we focus on accelerating probability-flow ODE sampling of pretrained diffusion/flow models under standard solvers (e.g., Euler and DPM-Solver++) while keeping model weights untouched.

\paragraph{Training-free Acceleration.}
Training-free acceleration methods reduce sampling latency by exploiting redundancy along different axes of diffusion/flow inference without further training.
A first line of work targets \emph{architectural} redundancy. Token reduction methods~\citep{bolya2023token, kim2024token} prune/merge redundant image tokens, while ToCa~\citep{zou2024accelerating, zhang2024token} combines adaptive token pruning with token-wise caching to further cut overhead.
For transformer backbones, attention-oriented approaches such as Sparse VideoGen~\citep{xi2025sparse,yang2025sparse} introduce spatio-temporal sparsity and semantic-aware permutation to accelerate video diffusion, and DiTFastAttn~\citep{yuan2024ditfastattn, zhang2025ditfastattnv2} compresses attention modules based on redundancies identified via a lightweight search.


A second line targets \emph{trajectory} redundancy along the denoising trajectory.
Feature-caching / reuse methods~\citep{ma2024deepcache, zhao2024real, wimbauer2024cache, liu2024timestep, chen2024delta, shenmd}
reuse intermediate activations across timesteps with author-designed cache policies. These methods achieves large speedups but \textbf{typically requiring layer-wise caches and thus nontrivial cache management}.
Beyond caching, more aggressive temporal methods attempt to \emph{forecast} future states using multi-step information and (often) higher-order rules to fill in several skipped steps:
TaylorSeer~\citep{liu2025taylor} forecasts future \emph{features} via Taylor expansion, and AB-Cache~\citep{yu2025ab} adopts Adams--Bashforth style multi-step updates to model step-to-step relations.
Orthogonal to forecasting, adaptive strategies such as AdaptiveDiffusion~\citep{ye2024training} and SADA~\citep{jiang2025sada} use input- and timestep-dependent \emph{adaptive step skipping} (and in SADA, also token-wise sparsification) to trade off sample quality and latency.


\textbf{How ZEUS differs from Prior works.} ZEUS provides \emph{training-free trajectory} acceleration via \emph{uniform} step skipping beyond $2\times$.
Since it only changes the sampling schedule, ZEUS leaves the denoiser architecture and per-call computation unchanged, making it complementary to token/attention sparsification, and it also avoids layer-wise activation caches and the associated cache policies/overheads.
Moreover, under uniform skipping with stride $>2$ (i.e., skipping at least two intermediate timesteps per update), higher-order multi-step updates, i.e., Taylorseer~\cite{liu2025taylor}, become unreliable because they are computed directly from previously approximated steps; accordingly, ZEUS does not use feature forecasting or higher-order solvers, and instead follows a simple deterministic rule, in contrast to adaptive skipping methods that rely on per-step heuristics or search.


%% file: texts/methodology.tex
\section{Background}
\label{sec:background}

Denoising generative models sample by integrating the probability flow ODE (PF--ODE), a deterministic process that preserves diffusion marginals~\citep{song2019generative,song2020score}.
The forward noising process is $\mathbf{x}_s = \alpha_s \mathbf{x}_0 + \sigma_s \boldsymbol{\epsilon}$ for $s \in [0,1]$, with data $\mathbf{x}_0 \sim p_{\mathrm{data}}$ and noise $\boldsymbol{\epsilon} \sim \mathcal{N}(0,I)$.
The PF--ODE reverses this process: $d\mathbf{x}_s = \mathbf{v}(\mathbf{x}_s, s)\,ds$, where the velocity $\mathbf{v}$ depends on the score $\nabla_{\mathbf{x}_s} \log q_s(\mathbf{x}_s)$ and $q_s$ is the marginal density of $\mathbf{x}_s$.
In practice, we discretize on a grid $\{s_t\}_{t=0}^{T}$ with $s_T = 1$ (noise) and $s_0 = 0$ (data) and apply a numerical solver; we write $\hat{\mathbf{x}}_t$ for the denoising state at step~$t$.

Different training conventions predict different--but equivalent--targets ($\boldsymbol{\epsilon}$, $\mathbf{x}_0$, $v$, the score, or a flow-matching velocity), which are linearly interconvertible given $\mathbf{x}_s$ and $(\alpha_s,\sigma_s)$. In ZEUS, we apply our approximation directly to the model output and do not assume any particular target form, so the method is parameterization-agnostic. We therefore use a unified notation: $\psi_\theta:\mathbb R^d\times[0,1]\to\mathbb R^d$ denotes the model output under the chosen parameterization. We write $\psi_t:=\psi_\theta(\hat{\mathbf{x}}_t,s_t)$ when the network is evaluated at step $t$ and $\hat\psi_t$ for an approximation used at reduced steps. Table~\ref{tab:parameterization_main} summarizes the corresponding prediction targets, and our analysis applies uniformly across these parameterizations.


Evaluating $\psi_t$ at each solver step typically dominates runtime latency, making denoiser calls the primary computational bottleneck.
ZEUS therefore aims to reduce the number of denoiser calls by sparsifying evaluations along the sampling trajectory while preserving sample fidelity.

\section{Proposed Method}
\label{sec:observation}


\begin{figure}[t]
  \centering
  \includegraphics[width=0.85\linewidth]{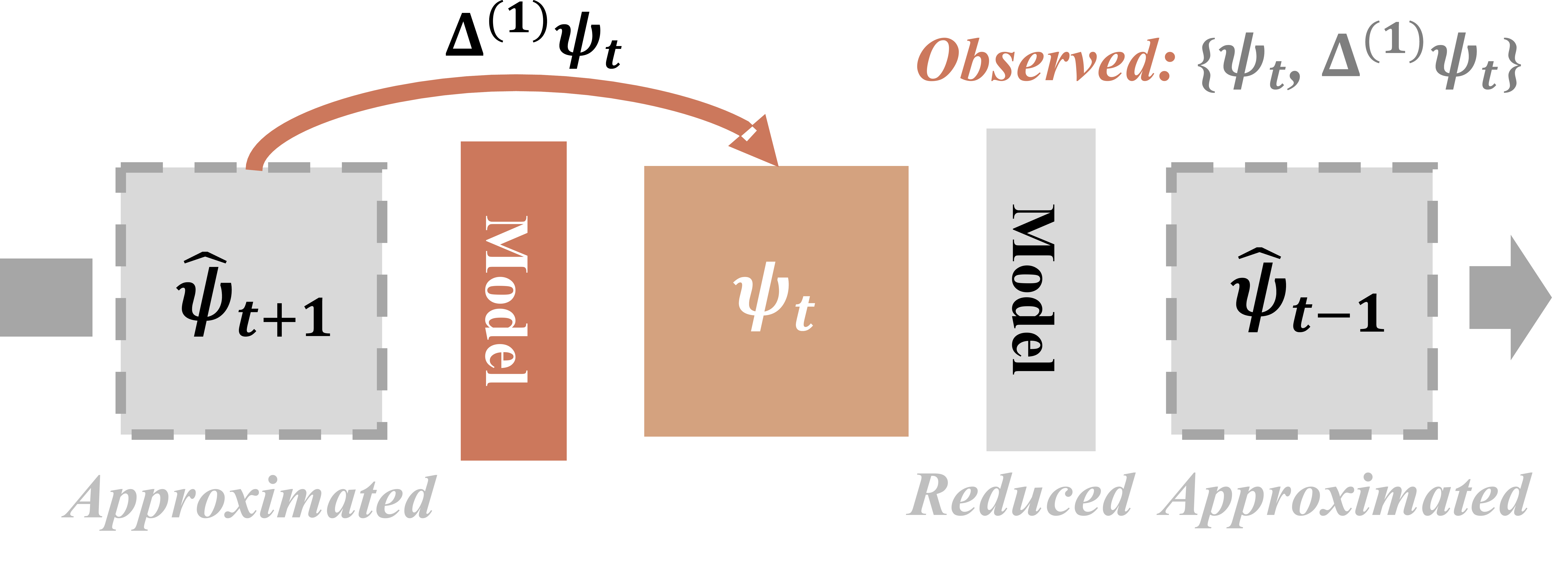}
  \caption{\textbf{Scarcity of full evaluations.}
  Under an aggressive acceleration ratio, we have limited denoiser evaluations, creating a scarcity of real information. 
  In this paper, we find that the (executed) denoising trajectory yields the \emph{observed, path-wise} information set
  $\{\psi_t,\ \Delta^{(1)}\psi_t\}$, where $\Delta^{(1)}\psi_t=\psi_t-\hat{\psi}_{t+1}$.}
  \label{fig:observed-set}
\end{figure}


We study training-free acceleration for discretized PF-ODE sampling through \best{trajectory step skipping}. Specifically, we use a uniform interleaved \(1{:}r\) schedule: within each block of \(r{+}1\) solver steps, we perform \best{one full denoiser evaluation} and skip the remaining \(r\) evaluations.
To execute these \best{reduced steps}, the solver still requires a denoiser output at each step; therefore, we must provide an approximation \(\hat{\psi}_t\) using a numerical predictor built from the trajectory information available at the full-evaluation steps.
This setting naturally raises two questions:
\textbf{(i)} at a full-evaluation step \(t\), what predictor most effectively uses the information observed along the denoising trajectory?
\textbf{(ii)} how can we approximate multiple (possibly consecutive) reduced steps while maintaining stable sampling?

\begin{figure*}[t]
  \centering
  \tiny
  \begin{subfigure}[t]{0.245\linewidth}
    \includegraphics[width=\linewidth, trim=2pt 2.5pt 2.2pt 12pt, clip]{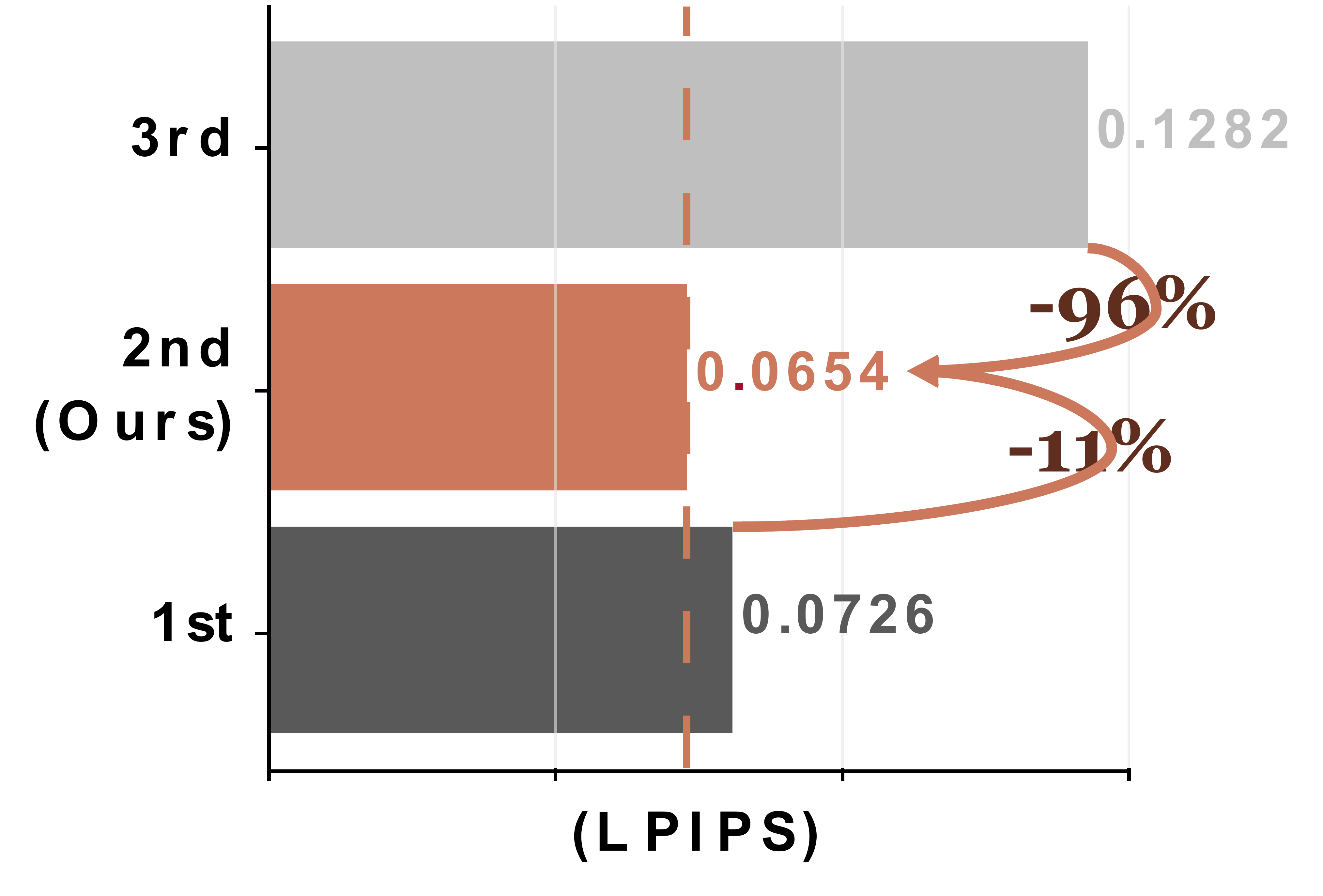}
    \caption{Perceptual deviation vs.\ predictor order (LPIPS~$\downarrow$)}
  \end{subfigure}\hfill
  \begin{subfigure}[t]{0.245\linewidth}
    \includegraphics[width=\linewidth, trim=2pt 2.5pt 2.2pt 12pt, clip]{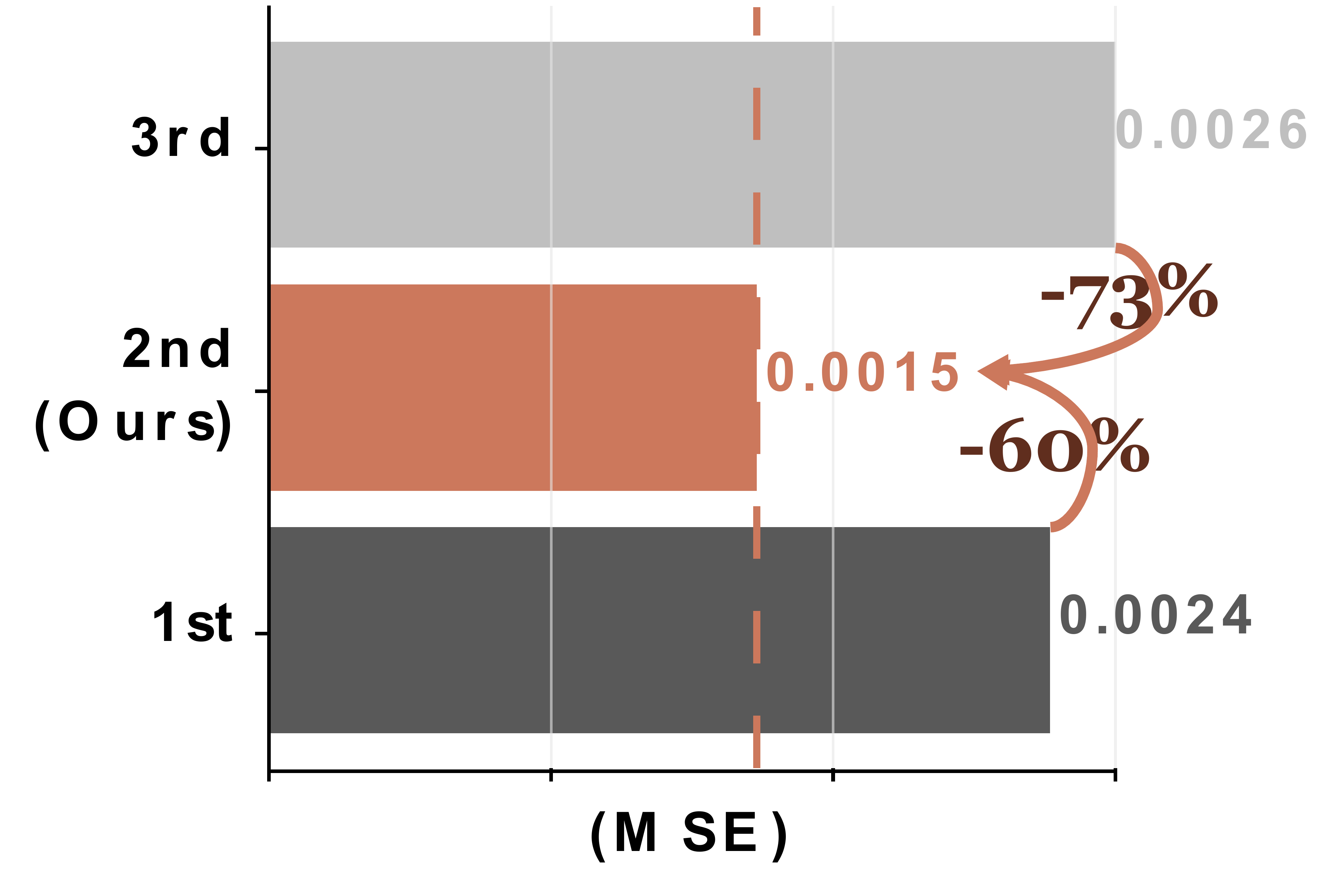}
    \caption{Per-step reconstruction error vs.\ predictor order (MSE~$\downarrow$)}
  \end{subfigure}\hfill
  \begin{subfigure}[t]{0.245\linewidth}
    \includegraphics[width=\linewidth, trim=3.75pt 2.5pt 2.2pt 12pt, clip]{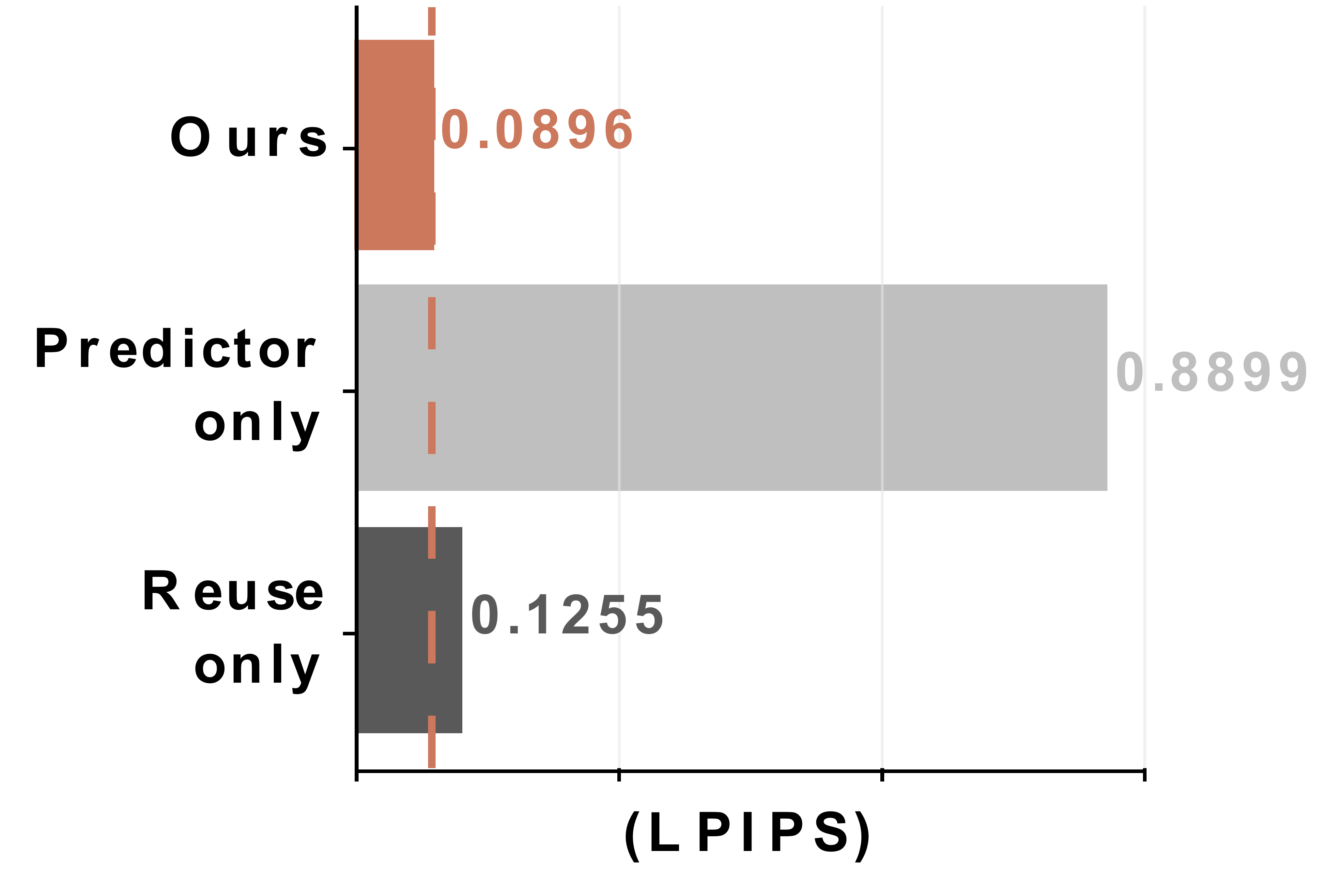}
    \caption{LPIPS for reuse-only, predictor-only, and ZEUS}
  \end{subfigure}\hfill
  \begin{subfigure}[t]{0.245\linewidth}
    \includegraphics[width=\linewidth, trim=2pt 2.5pt 2.2pt 10pt, clip]{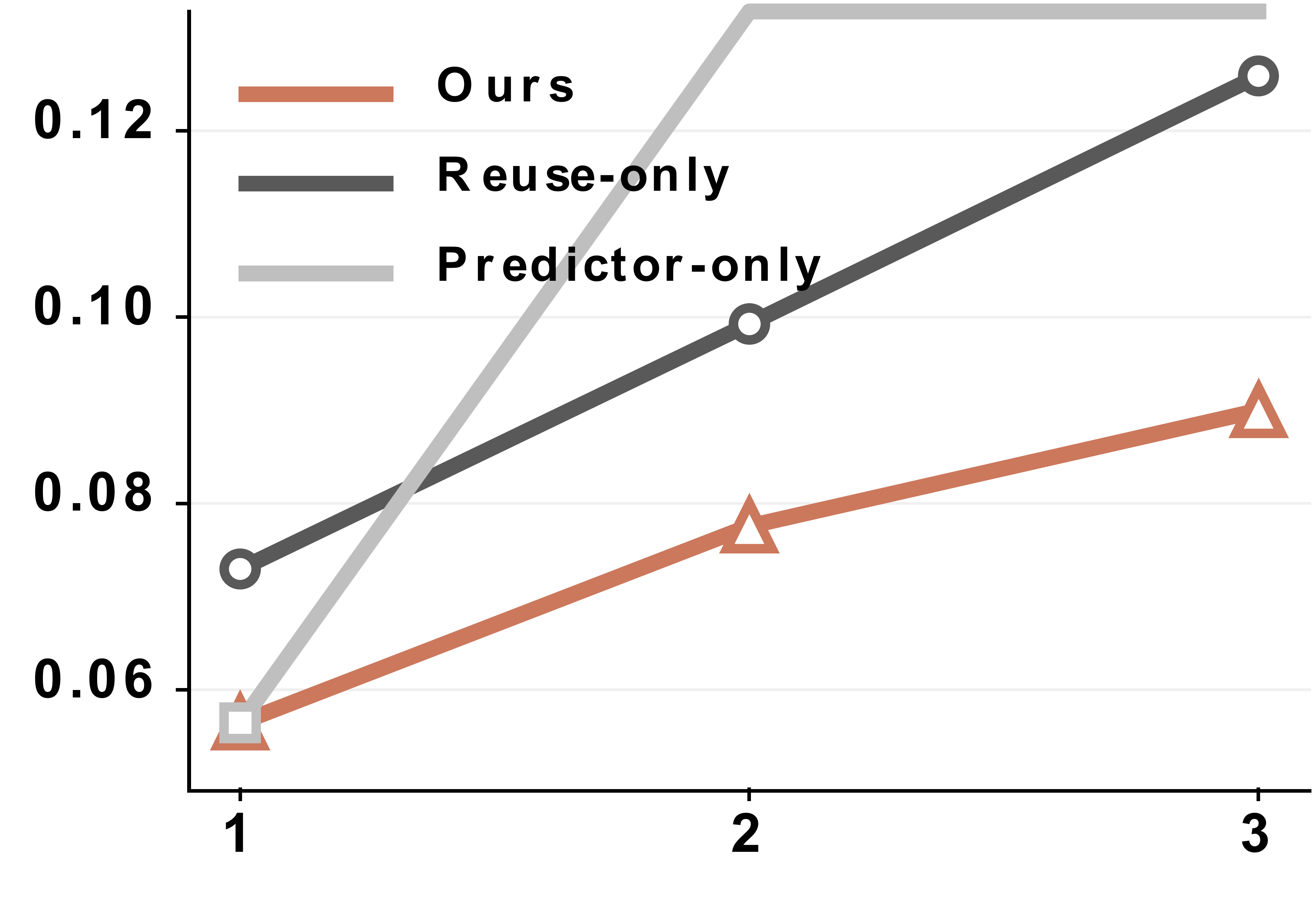}
    \caption{LPIPS vs.\ the number of consecutive reduced steps}
  \end{subfigure}

  \caption{
    \textbf{ZEUS ablations on SDXL.}
    We generate 1{,}000 samples from random MS-COCO 2017 prompts using DPM-Solver++ (50 steps).
    \textbf{(a,b) Predictor order under a uniform $1{:}1$ schedule.}
    \textbf{(c,d) Stability from reusing the observed information set under a uniform $1{:}3$ schedule.}
    This shows that ZEUS remains stable as the reduced-step run length increases.}
  \label{fig:zeus-ablations}
\end{figure*}


\subsection{The Effective Predictor Under Scarcity} 
\paragraph{Aggressive scarcity---1:$r$ ($r\geq2$) uniform interleaved schedule.}
Consider a full evaluation at step $t$ under a uniform interleaved $1{:}r$ schedule.
Because a full denoiser call happens only once every $r{+}1$ steps, the adjacent steps $t{\pm}1$ are reduced and their denoiser outputs $\hat{\psi}_{t\pm1}$ are provided by a predictor rather than computed.
As a result, this creates a fundamental \best{scarcity}: step $t$ is the only point in the local window with a fresh network output $\psi_t$; the remaining values are approximations derived from earlier evaluations.


\input{tables/rebuttal/order_ablation}

\paragraph{Higher numerical order does not imply higher accuracy under scarcity.}
We first investigate whether higher-order predictors yield lower reconstruction error under aggressive acceleration.
To isolate the effect of numerical order, we construct a $k$-point Lagrange extrapolation baseline: given the $k$ most recent full denoiser evaluations $\{\psi_{t}, \psi_{t+r+1}, \ldots, \psi_{t+k(r+1)}\}$, the baseline fits a $k$ order accurate predictor and approximate to $\hat\psi_{t-1} \cdots \hat\psi_{t-r}$.
This baseline mirrors the approximation scheme of TaylorSeer~\citep{liu2025taylor}, but operates on denoiser outputs $\psi_t$ directly rather than intermediate features.
We apply acceleration for the \textit{entire denoising trajectory} (no warm-up).

 
We show the results in Tab. \ref{tab:order_ablate}, which reports predictors of order $k \in \{2, 3\}$.
Two observations emerge.
First, surprisingly, the third-order predictor ($k{=}3$) consistently underperforms the second-order predictor ($k{=}2$) on both metrics, despite accessing one additional evaluation.
Second, error grows rapidly as the number of consecutive reduced steps $r$ increases, regardless of predictor order, therefore the effective predictor should use local information.

Therefore, rather than asking which predictor achieves the highest asymptotic order, we ask: which predictor \emph{exploits the most from a single full evaluation and the trajectory that produced it}?
To answer this, we first discuss what a predictor can actually access when fresh computation is scarce, and then derive the effective predictor that uses it.


\paragraph{The observed information set at full denoiser evaluation.}
At step $t$, the denoiser has computed $\psi_t = \psi_\theta(\hat{\mathbf{x}}_t, t)$.
But observe that the denoising state $\hat{\mathbf{x}}_t$ itself exists only because the solver advanced from $\hat{\mathbf{x}}_{t+1}$ using the signal $\hat\psi_{t+1}$.
Even though $\hat\psi_{t+1}$ was approximated, it is the \emph{committed input} that produced the current state.

This yields a key insight: the backward difference $\Delta^{(1)}\psi_t \coloneqq \psi_t - \hat\psi_{t+1}$ encodes information that $\psi_t$ alone cannot provide---the \textit{response} of the denoiser to the given (approximated) state.
Together, the pair $\{\psi_t, \Delta^{(1)}\psi_t\}$ constitutes what we call the \best{observed information set}: the complete causal record of the denoising trajectory at step $t$, available at zero marginal cost.

We emphasize that $\hat\psi_{t+1}$ need not be accurate for $\Delta^{(1)}\psi_t$ to be useful. 
Whatever approximation error $\hat\psi_{t+1}$ contains is already absorbed into $\hat{\mathbf{x}}_t$; the difference $\Delta^{(1)}\psi_t$ measures the model's response to that realized state. 
Discarding this signal---as reuse-only methods do---ignores the only trajectory-aware information available under scarcity.

\paragraph{The second-order predictor.}
Given the observed information set $\{\psi_t, \Delta^{(1)}\psi_t\}$, we can add the elements together, which forms a second-order predictor:
\begin{empheq}[box=\fcolorbox{gray!40}{gray!6}]{equation}
    \hat\psi_{t-1} = \psi_t + \Delta^{(1)}\psi_t = 2\psi_t - \hat\psi_{t+1}.
    \label{eq:backward-second}
\end{empheq}
This predictor uses exactly the two available signals: the fresh output and its backward difference.
In practice, this predictor requires only element-wise operations on two model output tensors, rather than layer-wise feature caching in \cite{liu2025faster} and \cite{liu2025taylor}.

We now show that \eqref{eq:backward-second} is optimal under scarcity.
The denoiser output along the sampling trajectory admits a decomposition $\psi_t = \phi(t) + \eta_t$, where $\phi$ is a smooth deterministic trend and $\eta_t$ is zero-mean noise (\cref{thm:decomp}).
Since the trajectory evolves smoothly, the trend $\phi$ is well-approximated as locally affine--and the coefficients $(2, -1)$ are the \emph{unique} weights that yield unbiased predictions for all affine $\phi$.
Under homoscedastic noise, these weights also minimize variance, making \eqref{eq:backward-second} the best linear unbiased estimator (BLUE) among two-point predictors (\cref{thm:blue}).

\textbf{Why not simpler or more complex?}
Reuse-only methods, equivalently first-order predictors, set $\hat\psi_{t-1} = \psi_t$, discarding slope information entirely.
Formally, any such one-point predictor suffers $\Omega(\Delta)$ bias (\cref{thm:one-point-lb}), which is an order of magnitude worse than the $\mathcal O(\Delta^2)$ achieved by \eqref{eq:backward-second}.

Meanwhile, a higher-order predictor offers no benefit under scarcity.
The additional inputs $\hat\psi_{t+2}, \hat\psi_{t+3}, \ldots$ are themselves approximations, so extending the predictor window induces more error, not more information.
Formally, $k$-th order Lagrange extrapolation amplifies variance by $\Omega(4^k)$ (\cref{thm:noise}), while the minimax bias rate remains $\Omega(\Delta^2)$ under $C^2$ smoothness.

We empirically show the gap via an ablation study on MS-COCO~\citep{lin2014microsoft}.
The surprising effectiveness of the second-order predictor holds when our numerical predictor takes local information at $t+1$ and $t+2$.
Fig.~\ref{fig:zeus-ablations}(a--b) demonstrates that both first- and third-order predictors degrade both MSE and LPIPS relative to the second-order predictor.
In sum, the \best{second}-order predictor \eqref{eq:backward-second} sits at the sweet spot: it fully exploits the observed
information set while avoiding the bias floor of reuse-only methods and the variance explosion of higher-order predictors.

\subsection{Approximating Multiple Consecutive Steps}
\paragraph{Stability matters when approximating consecutive steps.}
The second-order predictor \eqref{eq:backward-second} is optimal for a single reduced step.
However, aggressive speedups require $r \geq 2$ consecutive reduced steps, and the choice of how to approximate steps $j = 1, \ldots, r$ significantly affects stability.

We consider two natural baselines.
\textit{Predictor-only} chains the second-order rule.
Unrolling this recurrence yields $\hat\psi_{t-j}{=}(j{+}1)\cdot\psi_t{-}j\cdot\hat\psi_{t+1}$.
Because the coefficients grow linearly in $j$, any error in $\hat\psi_{t+1}$ is amplified at each step—the approximations progressively drift from the true trajectory (Fig.~\ref{fig:overshoot} left).
This baseline is used in \citet{liu2025taylor} and \citet{yu2025ab}.
\textit{Reuse-only} sets $\hat\psi_{t-j} = \psi_t$ for all $j$.
This avoids error amplification but ignores the slope information $\Delta^{(1)}\psi_t$ entirely, causing the approximations to lag behind the evolving trajectory and degrading fine details (Fig.~\ref{fig:overshoot} middle).
This baseline is used in \citet{ma2024deepcache} and \citet{ye2024training}.
In short, Predictor-only is precise but unstable; Reuse-only is stable but imprecise.
The question is whether we can achieve both.

\textbf{The interleaved approximation scheme achieves bounded variance with second-order bias.}
We resolve this tradeoff by \emph{reusing the observed information set} across all $r$ reduced steps.
Recall that the observed information set yields two values: the full evaluation $\psi_t$ and its second-order extrapolation $2\psi_t - \hat\psi_{t+1}$.
Rather than computing new approximations, we simply cycle between them:
\begin{empheq}[box=\fcolorbox{gray!40}{gray!6}]{equation}
    \hat\psi_{t-j} \;=\;
    \begin{cases}
        2\psi_t - \hat\psi_{t+1} & j \text{ odd}, \\[3pt]
        \psi_t & j \text{ even}.
    \end{cases}
    \label{eq:zigzag}
\end{empheq}
This \best{interleaved} pattern applies the second-order predictor once (at $j=1$), then resets to $\psi_t$, then reuses the same pair—never chaining predictions.
Because no new extrapolation occurs beyond $j=1$, variance remains bounded while second-order accuracy is preserved where it matters most.
A visualization of this approximation scheme is provided in Fig.~\ref{fig:teaser}

We now provide a bias–variance view for the approximation schemes.
Table~\ref{tab:bias-var} summarizes the theoretical tradeoffs (\cref{app:multistep-error}).
The \textit{Predictor-only} baseline achieves $\mathcal O(j^2\Delta^2)$ bias but suffers variance growth of $\Theta(j^2)\sigma^2$—the source of overshoot.
The \textit{Reuse-only} baseline maintains $\mathcal O(\sigma^2)$ variance but incurs $\mathcal O(j\Delta)$ bias.

The interleaved scheme combines the best of both \textbf{(i)} at $j=1$ (the most critical step), it coincides with the BLUE, achieving $\mathcal O(\Delta^2)$ bias; and \textbf{(ii)} for $j > 1$, resetting to $\psi_t$ caps variance at $\mathcal O(\sigma^2)$, avoiding $\Theta(j^2)$ growth.
The result is a schedule that is \emph{precise where precision matters} (small $j$) and \emph{stable where stability matters} (large $j$).
Fig.~\ref{fig:zeus-ablations}(c–d) confirms this: the interleaved scheme matches or outperforms both baselines across all reduction ratios.

\begin{figure}[t]
  \centering
  \includegraphics[width=\linewidth]{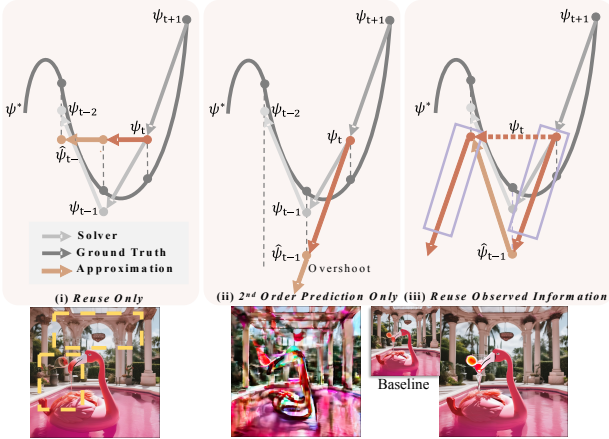}
  \caption{\textbf{Three approximation schemes.}
  We compare reuse-only, predictor-only, and ZEUS (reuse of the observed information pair).
  Dark gray: reference trajectory \(\psi^{\star}\).
  Light gray: solver-computed outputs \(\psi_t\).
  Crimson: approximated segments.
  \textbf{Left--Reuse only:} numerically stable but limited in expressivity; fine details erode (bottom).
  \textbf{Middle--Predictor only:} chaining second-order extrapolations overshoots without re-anchoring, producing artifacts (bottom).
  \textbf{Right--Reuse observed information (ZEUS):} alternating reuse of \(\{\psi_t,\ \psi_t+\Delta^{(1)}\psi_t\}\) prevents overshoot and preserves detail, yielding the best perceptual quality.}
  \label{fig:overshoot}
\end{figure}

%% file: tables/rebuttal/order_ablation.tex




\definecolor{zeusGreen}{RGB}{231,246,231}   
\definecolor{zeusYellow}{RGB}{255,247,214}  

\begin{table}[t]
  \centering
  \small
  \caption{\textbf{Ablation on numerical predictor order and input information.}
Second- and third-order Lagrange baselines extrapolate directly on the true previous full denoiser evaluation, whereas ZEUS (ours) exploits local solver information. Results are reported for different acceleration ratios
$r$ and predictor order $k$. This shows that (i) a higher-order numerical predictor does not necessarily come with higher accuracy and (ii) incorporating true but stale evaluation increases error under aggressive acceleration. The training-free strategy should exploit local information.}
  \label{tab:order_ablate}
  \setlength{\tabcolsep}{8pt}
  \renewcommand{\arraystretch}{1.15}
  \begin{tabular}{lccc}
    \toprule
    & \textbf{r=1} & \textbf{r=2} & \textbf{r=3} \\
    \midrule
    \rowcolor{zeusGreen} \textbf{2nd-order lagrange} & 0.1115 & 0.1907 & 0.2246 \\
    \rowcolor{zeusYellow}\textbf{3rd-order lagrange} & 0.1283 & 0.2431 & 0.2593 \\
    \rowcolor{zeusGreen} \textbf{ZEUS (ours)}              & 0.119  & 0.1745 & 0.2034 \\
    \textbf{speedup}                                & 1.80$\times$  & 2.41$\times$  & 2.81$\times$  \\
    \bottomrule
  \end{tabular}
\end{table}

%% file: texts/experiment.tex
\section{Experiment}
\label{sec:experiment}

\input{tables/image_exp}

\begin{table}[t]
\centering
\caption{Bias and variance for consecutive reduced steps. $\Delta$: step size; $\sigma^2$: noise variance.}
\label{tab:bias-var}
\small
\begin{tabular}{lcc}
\toprule
\textbf{Strategy} & \textbf{Bias} & \textbf{Variance} \\
\midrule
Predictor-only & $\mathcal O(j^2 \Delta^2)$ & $\Theta(j^2) \sigma^2$ \\
Reuse-only & $\mathcal O(j \Delta)$ & $\sigma^2$ \\
ZEUS & $\mathcal O(\Delta^2)$ at $j{=}1$; $\mathcal O(j\Delta)$ else & $\leq 5\sigma^2$ \\
\bottomrule
\end{tabular}
\end{table}

\subsection{Setup}

\paragraph{Models.}
We evaluate ZEUS on five text-conditioned generative models.
For text-to-image, we use Stable Diffusion~2.1 (865M parameters; $768{\times}768$; $v$-prediction), SDXL-Base (2.6B; $1024{\times}1024$; $\epsilon$-prediction), and Flux.1-dev (12B; $1024{\times}1024$; flow).
For text-to-video, we use Wan2.1-T2V-14B (14B; $480{\times}720{\times}81$; flow) and CogVideoX-v1.5 (5B; $480{\times}720{\times}49$; $\epsilon$-prediction),
where video resolution is H${\times}$W${\times}$T.
All experiments use official checkpoints, HuggingFace Diffusers, FP16 precision, and 50 sampling steps unless otherwise noted.
For images, we evaluate Euler (first-order) and DPM-Solver++ (second-order; \cite{lu2022dpm}); for videos, we use each model’s default solver.
All runs are performed on a single NVIDIA A100-80GB GPU.

\paragraph{Metrics.}
We assess the similarity of generated samples compared to full evaluation using the following metrics: PSNR ($\uparrow$), LPIPS~\citep{zhang2018unreasonable} ($\downarrow$), and SSIM ($\uparrow$).
For images, we additionally report FID~\citep{heusel2017gans} ($\downarrow$) between the baseline and accelerated sample sets to quantify distributional drift.
For videos, we report frame-averaged metrics; FID is omitted due to computational cost.
We report end-to-end wall-clock speed up.

\paragraph{Datasets.}
For text-to-image generation, we adopt the 5,000 prompts in the MSCOCO-2017 \citep{lin2014microsoft} validation set.
For text-to-video generation, we adopt prompts from the Penguin Benchmark, after optimizing them as provided by the VBench team.

\paragraph{Baselines.}
We compare against recent training-free acceleration methods.
We skip the method if it does not have an official implementation on the evaluated model.
For text-to-image generation, we compare against DeepCache~\citep{ma2024deepcache}, AdaptiveDiffusion~\citep{ye2024training}, SADA~\citep{jiang2025sada} on classical Stable Diffusion models, and ToCa~\citep{zou2024accelerating}, TaylorSeer~\citep{liu2025taylor}, TeaCache~\citep{liu2024timestep}, DiCache \citep{bu2025dicache}, and SADA on FLUX.1-dev.
For text-to-video generation, as we strictly mirror the setup in SVG series \citep{xi2025sparse, yang2025sparse}, we directly use the reported numbers of SVG, SVG2, TeaCache, DiTFastAttn \citep{yuan2024ditfastattn, zhang2025ditfastattnv2}, Minference \citep{jiang2024minference}, PAB \citep{zhao2024real}, and SpargeAttn \citep{zhang2025spargeattention}.

\paragraph{Implementation. }
ZEUS uses a $1{:}r$ reduction ratio, approximating $k$ consecutive steps after each full evaluation via interleave (Eq.~\ref{eq:zigzag}).
We test $r\in \{2, 3, 4\}$ (Medium/Fast/Turbo), applying ZEUS to the middle 70\% of the trajectory--the first 20\% and last 10\% use full evaluations due to higher curvature.
Medium is evaluated on all models; Fast on SDXL/Flux; Turbo on Flux only.

\subsection{Main Results}


\paragraph{Image results.}
Across SD~2.1, SDXL, and Flux, \textbf{ZEUS} consistently moves the speed--quality frontier outward under both Euler and DPM-Solver++ (Table~\ref{tab:zeus-main}). In classic Stable Diffusion backbones, ZEUS variants achieve the best or near-best end-to-end speedups while matching or improving perceptual similarity to full evaluation, indicating that the gains are not tied to a particular solver choice. On Flux (flow-matching), the improvement is larger: ZEUS traces a smooth Pareto curve from quality-oriented to throughput-oriented presets, and the quality-leaning setting remains on the perceptual frontier (LPIPS/FID) even when other methods are competitive on raw runtime. Together, these results support the central claim that a small, model-agnostic, training-free adjustment to step selection and cache-aware reconstruction can deliver state-of-the-art (or near-SOTA) efficiency without sacrificing perceptual realism. 

\input{tables/video_exp}

\paragraph{Video results.}
The same model-agnostic policy transfers to T2V: on both Wan~2.1 and CogVideoX-v1.5, ZEUS establishes a new quality--efficiency frontier without introducing video-specific kernels or modifying the backbone (Table~\ref{tab:main-video}). In balanced presets, ZEUS achieves the strongest perceptual similarity while still improving throughput over the strongest reported training-free baselines, and when pushed toward higher speed it remains close to the quality frontier rather than trading away fidelity disproportionately. This consistent shift in the speedup--LPIPS trade-off is also visible in Fig.~\ref{tab:ablation_flux}, supporting that ZEUS improves the latency-fidelity curve in a backbone-agnostic manner rather than exploiting quirks of a single model or evaluation protocol.



\subsection{Ablation Studies}

\paragraph{Few-step sampling.}
To stress-test ZEUS in the low-step regime, we evaluate few-step sampling across backbones and schedulers (Table~\ref{tab:ablation_flux}). As step budgets shrink, ZEUS continues to preserve (and often improve) perceptual similarity relative to the corresponding few-step baseline, while incurring only modest distributional drift as measured by FID, showing that the method remains effective even when inference becomes aggressively truncated. The same qualitative trend holds across Flux and SDXL and across Euler/DPM++ schedules, suggesting the benefit is not an artifact of a specific backbone or solver. 

\noindent\textbf{Solvers.}
Unless otherwise specified, we report results with standard diffusion samplers, including Euler (1st order) and DPM-Solver++ (2nd order). We additionally evaluate ZEUS with a higher-order solver setting, i.e., DPM-Solver++ (3rd order), and present the results in figure~\ref{fig:zeus-ablations}.

\paragraph{Higher-order solver ablation.}
ZEUS is solver-agnostic: it operates on denoiser outputs and can be paired with numerical solvers of different orders.
To verify this, we fix the model and sampling budget and vary only the solver order from Euler (1st order) to higher-order DPM-Solver++.
Table~\ref{tab:ablation_higher_order_solver} shows that increasing solver order consistently improves perceptual/distributional fidelity (lower LPIPS/FID), while PSNR and speedup remain essentially unchanged. This supports our claim that The second-order difference predictor is a simple, sufficient building block that remains effective across solver orders and improves as the underlying solver becomes more accurate.

\input{tables/ablation}

\input{tables/rebuttal/higher_order_solver_ablate}


%% file: tables/image_exp.tex
\begin{table*}[t]
\centering
\setlength{\tabcolsep}{8pt}        
\renewcommand{\arraystretch}{1.08} 
\caption{\textbf{Image quality vs.\ speed on SD-2, SDXL, and FLUX.}
Higher is better for PSNR/Speedup; lower is better for LPIPS/FID.
Per \emph{model+scheduler} block, best is \best{bold}, second best is \secondbest{underlined}.}
\label{tab:zeus-main}

\begin{tabular}{@{} l l l r r r r @{}}
\toprule
\textbf{Model} & \textbf{Scheduler} & \textbf{Method} &
\textbf{PSNR}~$\uparrow$ & \textbf{LPIPS}~$\downarrow$ & \textbf{FID}~$\downarrow$ & \textbf{Speedup}~$\uparrow$ \\
\midrule

\multirow{8}{*}{\makecell[l]{\textbf{SD-2}}}
  & \multirow{4}{*}{DPM++}
    & DeepCache           & 17.70 & 0.271 & 7.83 & \(1.43\times\) \\
  & & AdaptiveDiffusion   & 24.30 & \secondbest{0.100} & \secondbest{4.35} & \(1.45\times\) \\
  & & SADA                & \best{26.34} & \best{0.094} & \best{4.02} & \secondbest{\(1.80\times\)} \\
  & & ZEUS-Medium         & \secondbest{25.72} & 0.104 & 4.46 & \best{\(1.85\times\)} \\
\cmidrule(l){2-7}
  & \multirow{4}{*}{Euler}
    & DeepCache           & 18.91 & 0.239 & 7.40 & \(1.45\times\) \\
  & & AdaptiveDiffusion   & 21.94 & 0.173 & 7.58 & \best{\(1.89\times\)} \\
  & & SADA                & \best{26.25} & \best{0.100} & \best{4.26} & \(1.81\times\) \\
  & & ZEUS-Medium         & \secondbest{25.37} & \secondbest{0.118} & \secondbest{5.06} & \secondbest{\(1.86\times\)} \\
\midrule

\multirow{10}{*}{\makecell[l]{\textbf{SDXL}}}
  & \multirow{5}{*}{DPM++}
    & DeepCache           & 21.35 & 0.255 & 8.48 & \(1.74\times\) \\
  & & AdaptiveDiffusion   & 26.16 & \secondbest{0.125} & 4.59 & \(1.65\times\) \\
  & & SADA                & \best{29.36} & \best{0.084} & \best{3.51} & \(1.86\times\) \\
  & & ZEUS-Medium         & \secondbest{29.17} & \best{0.084} & \secondbest{3.59} & \secondbest{\(1.87\times\)} \\
  & & ZEUS-Fast           & 26.38 & 0.129 & 5.39 & \best{\(1.93\times\)} \\
\cmidrule(l){2-7}
  & \multirow{5}{*}{Euler}
    & DeepCache           & 22.00 & 0.223 & 7.36 & \best{\(2.16\times\)} \\
  & & AdaptiveDiffusion   & 24.33 & 0.168 & 6.11 & \secondbest{\(2.01\times\)} \\
  & & SADA                & \best{28.97} & \best{0.093} & \best{3.76} & \(1.85\times\) \\
  & & ZEUS-Medium         & \secondbest{28.66} & \secondbest{0.095} & \secondbest{3.87} & \(1.85\times\) \\
  & & ZEUS-Fast           & 25.15 & 0.153 & 6.47 & \(1.93\times\) \\
\midrule

\multirow{8}{*}{\makecell[l]{\textbf{FLUX}}}
  & \multirow{8}{*}{Euler (Flow)}
    & TeaCache            & 19.14 & 0.216 & 4.89 & \(2.00\times\) \\
  & & SADA                & \secondbest{29.44} & \secondbest{0.060} & \secondbest{1.95} & \(2.02\times\) \\
  & & ToCa                & 17.70 & 0.352 & 8.84 & \(1.52\times\) \\
  & & TaylorSeer          & 15.36 & 0.430 & 10.08 & \secondbest{\(3.13\times\)} \\
  & & DiCache             & 22.39 & 0.270 & /    & \best{\(3.22\times\)} \\
  & & ZEUS-Medium         & \best{30.19} & \best{0.047} & \best{1.53} & \(2.09\times\) \\
  & & ZEUS-Fast           & 26.77 & 0.079 & 2.49 & \(2.47\times\) \\
  & & ZEUS-Turbo          & 21.80 & 0.171 & 4.52 & \best{\(3.22\times\)} \\
\bottomrule
\end{tabular}
\end{table*}

%% file: tables/video_exp.tex

\begin{table*}[t]
\centering
\setlength{\tabcolsep}{8pt}
\renewcommand{\arraystretch}{0.82}
\caption{\textbf{Quality \& efficiency on video generation.} Higher is better for PSNR/SSIM/Speedup; lower is better for LPIPS.
Per-\emph{model} best is \best{bold}, second-best is \secondbest{underlined}.}
\label{tab:main-video}
\begin{tabular}{llcccc}
\toprule
\textbf{Model} & \textbf{Method} & \textbf{PSNR}~$\uparrow$ & \textbf{SSIM}~$\uparrow$ & \textbf{LPIPS}~$\downarrow$ & \textbf{Speedup}~$\uparrow$ \\
\midrule

\multicolumn{1}{l}{\textbf{Wan 2.1}} &  & & & & \\
\cmidrule(lr){1-6}
& SpargeAttn     & 20.52 & 0.623  & 0.343 & 1.44$\times$ \\
& SVG            & 22.99 & 0.785  & 0.199 & 1.58$\times$ \\
& SVG2           & \secondbest{25.81} & \best{0.854} & \secondbest{0.138} & 1.58$\times$ \\
& SVG2-Turbo     & 23.68 & 0.789  & 0.196 & \secondbest{1.89$\times$} \\
\rowcolor{OurHL}
& Ours & \best{29.00} & \secondbest{0.846} & \best{0.117} & 1.76$\times$ \\
\rowcolor{OurHL}
& Ours-Fast      & 26.59 & 0.792 & 0.179 & \best{2.24$\times$} \\
\addlinespace[2pt]


\multicolumn{1}{l}{\textbf{CogVideoX-v1.5 }} & & & & & \\
\cmidrule(lr){1-6}
& DiTFastAttn    & 23.20 & 0.741  & 0.256 & 1.56$\times$ \\
& Minference     & 22.45 & 0.691  & 0.304 & 1.48$\times$ \\
& PAB            & 22.49 & 0.740  & 0.234 & 1.41$\times$ \\
& SVG            & \secondbest{29.99} & \best{0.910} & 0.112 & \best{2.28$\times$} \\
\rowcolor{OurHL}
& Ours & \best{32.50} & \secondbest{0.893} & \best{0.060} & 1.71$\times$ \\
\rowcolor{OurHL}
& Ours-Fast      & 30.97 & 0.876  & \secondbest{0.073} & \secondbest{2.00$\times$} \\
\bottomrule
\end{tabular}
\end{table*}

%% file: tables/ablation.tex
\begin{table}[t]
\centering
\scriptsize
\caption{Ablation study on few-step sampling across schedulers. Results on MS-COCO 2017.}
\label{tab:ablation_flux}
\resizebox{\linewidth}{!}{ 
\begin{tabular}{lcccc}
\toprule
\multicolumn{5}{c}{\textbf{Flux}} \\
\midrule
Scheduler & Steps & LPIPS $\downarrow$ & FID $\downarrow$ & Speedup \\
\midrule
Euler  & 50 & 0.047 & 1.53 & 2.09$\times$ \\
       & 25 & 0.040 & 2.65 & 1.58$\times$ \\
       & 15 & 0.036 & 2.47 & 1.22$\times$ \\
\midrule
\multicolumn{5}{c}{\textbf{SDXL}} \\
\midrule
Euler  & 50 & 0.095 & 3.87 & 1.85$\times$ \\
       & 25 & 0.064 & 4.98 & 1.46$\times$ \\
       & 15 & 0.065 & 4.72 & 1.20$\times$ \\
\midrule
DPM++  & 50 & 0.084 & 3.59 & 1.87$\times$ \\
       & 25 & 0.066 & 5.28 & 1.45$\times$ \\
       & 15 & 0.066 & 4.87 & 1.20$\times$ \\
\bottomrule
\end{tabular}
}
\end{table}

%% file: tables/rebuttal/higher_order_solver_ablate.tex

\begin{table}[t]
  \centering
  \small
  \caption{Ablation on higher-order solvers.}
  \label{tab:ablation_higher_order_solver}
  \setlength{\tabcolsep}{7pt}
  \renewcommand{\arraystretch}{1.15}
  \begin{tabular}{l S[table-format=2.2] S[table-format=1.3] S[table-format=1.2] c}
    \toprule
    \textbf{Solver} &
    {\textbf{PSNR} $\uparrow$} &
    {\textbf{LPIPS} $\downarrow$} &
    {\textbf{FID} $\downarrow$} &
    {\textbf{Speedup} $\uparrow$} \\
    \midrule
    Euler-1  & 28.66 & 0.095 & 3.87 & \textbf{1.85}$\times$ \\
    DPM++-2  & 29.17 & 0.084 & 3.59 & \textbf{1.87}$\times$ \\
    DPM++-3  & 29.05 & 0.079 & 3.28 & \textbf{1.88}$\times$ \\
    \bottomrule
  \end{tabular}
\end{table}

%% file: texts/conclusions.tex
\section{Conclusion}
In this paper, we introduce ZEUS: a minimal, training-free, method-agnostic plug-in that consistently shifts the speed–fidelity Pareto frontier across five backbones and two solvers—achieving up to 3.2$\times$ end-to-end speedup on Flux while improving LPIPS/FID/PSNR.
ZEUS addresses the scarcity of fresh computation in an ambitious acceleration scenario.
Reusing and leveraging the observed information set from this constraint, yielding either better quality at the same cost or higher speed at comparable quality.
ZEUS conveys a counterintuitive but compelling message: \textit{\textbf{accelerating diffusion is as easy as a second-order predictor.}}



%% file: texts/impact_statement.tex
\section*{Impact Statement}

This work proposes ZEUS, a training-free acceleration method for diffusion/flow-based generative models that reduces sampling latency by skipping denoiser evaluations and approximating the missing outputs with a simple second-order predictor and an interleaved schedule, without modifying model architectures or weights and with essentially zero additional overhead. 

The primary positive impact is improved efficiency: lowering inference cost can reduce energy per generated sample, improve accessibility for researchers and practitioners with limited compute budgets, and enable more interactive applications (e.g., rapid prototyping, iterative design, or real-time creative tools) for both image and video generation settings. 

At the same time, accelerating generative models can amplify existing societal risks associated with high-fidelity synthetic media. Faster generation lowers the marginal cost of producing large volumes of outputs, potentially increasing misuse such as misinformation, impersonation or non-consensual synthetic content, and large-scale spam. Because ZEUS is model-agnostic and training-free, it may be applied to pretrained checkpoints that differ widely in their safety mitigations; the method itself does not add content safeguards or change the underlying model’s alignment properties. We therefore encourage pairing acceleration with responsible deployment practices, including adherence to model and dataset licenses, platform policies, provenance and transparency mechanisms, safety filtering and monitoring, and appropriate rate limiting for public-facing systems.

Finally, while ZEUS reduces per-sample compute, a rebound effect is possible if lower latency substantially increases total generation volume. Users should evaluate efficiency gains alongside safety constraints and application context, especially in high-stakes domains where approximation artifacts or biased generations may cause harm.

%% file: texts/appendix.tex
\setcounter{equation}{0}
\setcounter{figure}{0}
\setcounter{table}{0}
\renewcommand\thefigure{A.\arabic{figure}}
\renewcommand\theequation{A.\arabic{equation}}
\renewcommand\thetable{A.\arabic{table}}

\section{Mathematical Foundations}
\label{sec:math_foundation}
\subsection{Notation}
In this section, we firstly formalize all the specific notation used in the paper.

\begin{definition}
    \label{def:prob-space}
    Let $(\Omega,\mathcal F,\Pr)$ be the probability space that generates the random pair $(x_{0},\boldsymbol{\epsilon})$ and hence the noisy latent \(x_{t}=\alpha_{t}\,x_{0}+\sigma_{t}\boldsymbol{\epsilon}\)
    at any $t\!\in\{1,\dots,T\}$. For a fixed, deterministic network
    \(
    \psi_\theta(\cdot,\,\cdot):\mathbb R^{d}\!\times[0,1]\!\to\!\mathbb R^{d}
    \),
    which is the network training result for $\psi_t$ like $\boldsymbol{\epsilon}$ or $v$ in the forward process, define the random output
    \(
    \psi_{\theta, t}:=\psi_\theta(x_{t},t)\in L^{2}(\Omega)
    \) and define the output in inference process 
     \(
    \hat{\psi}_{t}:=\psi_\theta(\hat{x}_{t},t).
    \)
\end{definition}

\paragraph{General forward process.}
Let $(\Omega,\mathcal F,\Pr)$ be a probability space.  
A data sample is represented by a random vector 
\(
\mathbf{x}_{0}:\Omega\to\mathbb R^{d}
\)
with distribution \(p_{\rm data}\).  
Let \(\boldsymbol{\epsilon}:\Omega\to\mathbb R^{d}\) be an independent standard Gaussian noise, i.e.\ \(\boldsymbol{\epsilon}\sim\mathcal N(\mathbf{0},\mathbf{I}_{d})\). 
We distinguish continuous time $s\in[0,1]$ and a uniform discrete grid
$s_t:=t/T$ for $t\in\{0,1,\ldots,T\}$ with step size $\Delta:=1/T$. For each $s\in[0,1]$, define deterministic schedule functions $\alpha_s\in(0,1],\;\sigma_s\in(0,1]$.

The forward latent at time \(s\) is defined as
\begin{equation}\label{eq:forward-process}
\mathbf{x}_s \;:=\; \alpha_s \,\mathbf{x}_0 + \sigma_s \,\boldsymbol{\epsilon},
\qquad
\mathbf{x}_s \in L^2(\Omega;\mathbb R^d).
\end{equation}

Equivalently, conditioned on $\mathbf{x}_0$, the marginal distribution is Gaussian:
\[
q(\mathbf{x}_s \mid \mathbf{x}_0) \;=\; 
\mathcal N\!\bigl(\alpha_s \mathbf{x}_0,\;\sigma_s^2 \mathbf{I}_d \bigr).
\]

Thus the diffusion forward process is the family 
\(
\{\mathbf{x}_s : s\in[0,1]\},
\)
with discrete samples \(\{\mathbf{x}_t : t=0,\dots,T\}\) obtained by evaluating at time grid points.

\paragraph{Reverse process.}
The forward process \(\{\mathbf{x}_s : s\in[0,1]\}\) defined in \eqref{eq:forward-process} is Markovian.  
In particular, for any \(0\leq s'<s\leq 1\), the conditional distribution of $\mathbf{x}_{s'}$ given $(\mathbf{x}_s,\mathbf{x}_0)$ is Gaussian:
\begin{equation}\label{eq:reverse-kernel}
q(\mathbf{x}_{s'} \mid \mathbf{x}_s,\mathbf{x}_0)
\;=\;
\mathcal N\!\Bigl(
\mu_{s',s}(\mathbf{x}_s,\mathbf{x}_0),
\;\;\Sigma_{s',s}
\Bigr),
\end{equation}
with mean
\[
\mu_{s',s}(\mathbf{x}_s,\mathbf{x}_0)
= \frac{\alpha_{s'}}{\alpha_s}\,\mathbf{x}_s
+ \Bigl(\alpha_0 - \frac{\alpha_{s'}}{\alpha_s}\alpha_s\Bigr)\mathbf{x}_0
= \frac{\alpha_{s'}}{\alpha_s}\,\mathbf{x}_s
+ \Bigl(1-\frac{\alpha_{s'}}{\alpha_s}\Bigr)\mathbf{x}_0,
\]
and covariance
\[
\Sigma_{s',s} = 
\Bigl(\sigma_{s'}^2 - \frac{\alpha_{s'}^2}{\alpha_s^2}\,\sigma_s^2\Bigr)\mathbf{I}_d .
\]

Equivalently, marginalizing out $\mathbf{x}_0$, the reverse-time transition kernel can be expressed as
\begin{equation}\label{eq:reverse-marginal}
q(\mathbf{x}_{s'} \mid \mathbf{x}_s)
= \int q(\mathbf{x}_{s'} \mid \mathbf{x}_s,\mathbf{x}_0)\,p(\mathbf{x}_0 \mid \mathbf{x}_s)\,d\mathbf{x}_0,
\end{equation}
which is generally intractable.  
The role of the neural network is precisely to approximate the conditional dependence on $\mathbf{x}_0$ by predicting either the noise $\boldsymbol{\epsilon}$, the clean sample $\mathbf{x}_0$, or equivalent parameterizations.

\paragraph{Training objective.}
The network is trained by conditional regression: for a chosen
ground-truth target $\psi_0(\mathbf{x}_0,\boldsymbol{\epsilon},s)$, 
we minimize
\[
\min_\theta\;
\mathbb{E}_{\mathbf{x}_0\sim p_{\rm data},\,\boldsymbol{\epsilon}\sim\mathcal N(0,I),\,s\sim\mathcal U[0,1]}
\bigl[
\ell\!\bigl(\psi_\theta(\mathbf{x}_s,s),\;\psi_0(\mathbf{x}_0,\boldsymbol{\epsilon},s)\bigr)
\bigr],
\]
where $\ell$ is typically the squared error. 

\paragraph{Unified network parameterization.}
To express all variants within a single notation, we define 
\[
\psi_\theta:\mathbb{R}^d\times[0,1]\to\mathbb{R}^d, \qquad
\psi_\theta(\mathbf{x}_s,s)\;\text{trained to predict}\;\psi_0(\mathbf{x}_0,\boldsymbol{\epsilon},s).
\]
Here $\psi_0$ is chosen from a finite family of equivalent targets, e.g. 
\[ \psi_0 \in \bigl\{\; \boldsymbol{\epsilon},\; \mathbf{x}_0,\; \alpha_s\boldsymbol{\epsilon}-\sigma_s\mathbf{x}_0,\; \nabla_{\mathbf{x}_s}\log q_s(\mathbf{x}_s),\; \boldsymbol{\epsilon}-\mathbf{x}_0 \;\bigr\}, \]
corresponding to $\epsilon$-prediction, $\mathbf{x_0}$-prediction, $v$-prediction, 
score-prediction, and flow-matching.
Thus, different implementations correspond to linearly related instances
of the same abstract map $\psi_\theta$, which enables uniform analysis of
loss functions and inference rules.

\paragraph{Continuous-time probability flow ODE.}
The forward process $\{\mathbf{x}_s : s\in[0,1]\}$ in \eqref{eq:forward-process} admits a continuous-time formulation
as a linear Itô SDE:
\begin{equation}\label{eq:forward-sde}
d\mathbf{x}_s \;=\; f(s)\,\mathbf{x}_s\,ds + g(s)\,d\mathbf{w}_s,
\end{equation}
where $\mathbf{w}_s$ is a standard Wiener process in $\mathbb{R}^d$, and the drift/diffusion coefficients
$(f(s),g(s))$ are determined by the schedules $(\alpha_s,\sigma_s)$.  
Concretely, the marginal law of $\mathbf{x}_s$ given $\mathbf{x}_0$ is Gaussian with mean $\alpha_s\mathbf{x}_0$
and variance $\sigma_s^2\mathbf{I}_d$.

\medskip

Following \citep{song2020score,song2019generative,chen2022sampling}, the reverse-time SDE that generates the same marginal distributions runs backward from $s=1$ to $s=0$:
\begin{equation}\label{eq:reverse-sde}
d\mathbf{x}_s \;=\; \bigl[f(s)\,\mathbf{x}_s - g(s)^2\,\nabla_{\mathbf{x}_s}\log q_s(\mathbf{x}_s)\bigr]\,ds + g(s)\,d\bar{\mathbf{w}}_s,
\end{equation}
where $q_s$ denotes the density of $\mathbf{x}_s$, and $\bar{\mathbf{w}}_s$ is a standard Wiener process running backward in time.

\medskip
The \emph{probability flow ODE} (PF-ODE) is the deterministic counterpart of \eqref{eq:reverse-sde}, obtained by removing the stochastic term:
\begin{equation}\label{eq:pf-ode}
d\mathbf{x}_s \;=\; \bigl[f(s)\,\mathbf{x}_s - \tfrac{1}{2} g(s)^2\,\nabla_{\mathbf{x}_s}\log q_s(\mathbf{x}_s)\bigr]\,ds.
\end{equation}
This ODE preserves the exact marginal distributions $\{q_s\}_{s\in[0,1]}$ of the forward process, and therefore provides a deterministic generative sampling procedure.


\subsection{Unification of Prediction Objectives}
\label{sec:theorem}
\paragraph{Consistency of the training objective.}
Recall that the network is trained by conditional regression:
\[
\min_\theta\;
\mathbb{E}_{\mathbf{x}_0\sim p_{\rm data},\,\boldsymbol{\epsilon}\sim\mathcal N(0,I),\,s\sim\mathcal U[0,1]}
\bigl[
\ell\!\bigl(\psi_\theta(\mathbf{x}_s,s),\;\psi_0(\mathbf{x}_0,\boldsymbol{\epsilon},s)\bigr)
\bigr],
\]
where $\mathbf{x}_s = \alpha_s \mathbf{x}_0 + \sigma_s \boldsymbol{\epsilon}$ and $\ell$ is the squared error loss.

\begin{theorem}[Optimal predictor under $L^2$ training]
\label{thm:l2-optimality}
Let $\ell(a,b) = \|a-b\|_2^2$. 
Define the population risk
\[
\mathcal L(\theta)
:= \mathbb{E}_{\mathbf{x}_0,\boldsymbol{\epsilon},s}
\bigl[
\|\psi_\theta(\mathbf{x}_s,s) - \psi_0(\mathbf{x}_0,\boldsymbol{\epsilon},s)\|_2^2
\bigr].
\]
Then any minimizer $\psi^*$ of $\mathcal L$ satisfies, for all $(\mathbf{x}_s,s)$,
\begin{equation}
\psi^*(\mathbf{x}_s,s)
= \mathbb{E}\!\left[\,\psi_0(\mathbf{x}_0,\boldsymbol{\epsilon},s) \;\middle|\; \mathbf{x}_s, s \,\right].
\label{eq:l2-conditional-expectation}
\end{equation}
In other words, in the $L^2$ sense, training recovers the conditional expectation of the regression target $\psi_0$ given the noisy input $(\mathbf{x}_s,s)$. Let the hypothesis class be all measurable maps with finite second moment; equivalently, consider the Bayes risk minimization.”
\end{theorem}

\begin{proof}
For fixed $(\mathbf{x}_s,s)$, define the conditional distribution 
\[
p(\mathbf{x}_0,\boldsymbol{\epsilon} \mid \mathbf{x}_s,s).
\]
The contribution of $(\mathbf{x}_s,s)$ to the expected loss is
\[
\mathbb{E}_{\mathbf{x}_0,\boldsymbol{\epsilon}\mid \mathbf{x}_s,s}
\Bigl[
\|\psi_\theta(\mathbf{x}_s,s) - \psi_0(\mathbf{x}_0,\boldsymbol{\epsilon},s)\|_2^2
\Bigr].
\]
This is a convex quadratic in $\psi_\theta(\mathbf{x}_s,s)$, uniquely minimized at
\[
\psi^*(\mathbf{x}_s,s)
= \mathbb{E}\!\left[\,\psi_0(\mathbf{x}_0,\boldsymbol{\epsilon},s) \;\middle|\; \mathbf{x}_s, s \,\right].
\]
Therefore, the global risk minimizer $\psi^*$ coincides with the conditional expectation~\eqref{eq:l2-conditional-expectation}.
\end{proof}

\paragraph{Discussion.}
The theorem shows that, under exact optimization of the $L^2$ objective, the learned network $\psi_\theta$ does not in general recover the ground-truth target $\psi_0(\mathbf{x}_0,\boldsymbol{\epsilon},s)$ pointwise.  
Instead, it recovers its conditional expectation given the accessible input $(\mathbf{x}_s,s)$.  
Thus, the choice of $\psi_0$ (noise, clean data, or equivalent parameterizations) directly determines which conditional expectation is realized by the trained model.

\paragraph{Unified network parameterization.}
We now establish that all common training targets in diffusion-type models
are instances of the same conditional regression principle. 

\begin{theorem}[Equivalence of parameterizations]
\label{thm:unified-param}
Let the forward process follow Definition~A.1, i.e.
$\mathbf{x}_s = \alpha_s \mathbf{x}_0 + \sigma_s \boldsymbol{\epsilon}$ 
with $\boldsymbol{\epsilon}\sim\mathcal N(0,I)$.  
Fix a target functional $\psi_0(\mathbf{x}_0,\boldsymbol{\epsilon},s)$ from a finite set of admissible forms.  
Suppose the network is trained with the $L^2$ objective
\[
\mathcal L(\theta)
=\mathbb E_{\mathbf{x}_0,\boldsymbol{\epsilon},s}
\bigl[\|\psi_\theta(\mathbf{x}_s,s)-\psi_0(\mathbf{x}_0,\boldsymbol{\epsilon},s)\|^2\bigr].
\]
Then, in the limit of exact optimization, the optimal predictor satisfies
\[
\psi^*(\mathbf{x}_s,s)
=\mathbb E\!\left[\,\psi_0(\mathbf{x}_0,\boldsymbol{\epsilon},s)\;\middle|\;\mathbf{x}_s,s\,\right].
\]
Moreover, for each admissible $\psi_0$, there exists deterministic coefficients $(a_s,b_s)$ such that
the clean data is exactly recovered by
\begin{equation}
\hat{\mathbf{x}}_0^{(s)} 
= a_s \mathbf{x}_s + b_s \psi^*(\mathbf{x}_s,s).
\label{eq:unified-linear}
\end{equation}
Thus, all parameterizations are equivalent in expressive power: they differ only in the choice of regression target $\psi_0$ and in the reconstruction formula~\eqref{eq:unified-linear}.
\end{theorem}

\begin{proof}
From Theorem~\ref{thm:l2-optimality}, the $L^2$ minimizer satisfies
\[
\psi^*(\mathbf{x}_s,s)\;=\;\mathbb{E}\!\left[\psi_0(\mathbf{x}_0,\boldsymbol{\epsilon},s)\,\middle|\,\mathbf{x}_s,s\right].
\]
By construction, the forward process is linear-Gaussian:
\begin{equation}
\label{eq:fwd}
\mathbf{x}_s \;=\; \alpha_s\,\mathbf{x}_0 \;+\; \sigma_s\,\boldsymbol{\epsilon},
\qquad 
\boldsymbol{\epsilon}\sim\mathcal{N}(\mathbf{0},\mathbf{I}),\ \ \boldsymbol{\epsilon}\perp\!\!\!\perp\mathbf{x}_0.
\end{equation}
Each admissible training target is an affine (here linear) functional of $(\mathbf{x}_0,\boldsymbol{\epsilon})$.  
We unify them by writing
\begin{equation}
\label{eq:psi0-uv}
\psi_0(\mathbf{x}_0,\boldsymbol{\epsilon},s)\;=\;u_s\,\boldsymbol{\epsilon}\;+\;v_s\,\mathbf{x}_0,
\end{equation}
where $(u_s,v_s)$ are scalar (or diagonal, coordinate-wise) coefficients depending only on $s$.
Combining \eqref{eq:fwd}--\eqref{eq:psi0-uv},
\[
\begin{bmatrix}
\mathbf{x}_s\\[2pt]
\psi_0
\end{bmatrix}
\;=\;\begin{bmatrix}
\alpha_s & \sigma_s\\[2pt]
v_s & u_s
\end{bmatrix}
\begin{bmatrix}
\mathbf{x}_0\\[2pt]
\boldsymbol{\epsilon}
\end{bmatrix}:=~\mathbf{M}_s\begin{bmatrix}
\mathbf{x}_0\\[2pt]
\boldsymbol{\epsilon}
\end{bmatrix}.
\]
Assume $\det(\mathbf{M}_s)\neq 0$, i.e.\ $\Delta_s:=\alpha_s u_s-\sigma_s v_s\neq 0$.  
Then $\mathbf{M}_s$ is invertible and we have the \emph{algebraic identity}
\begin{equation}
\label{eq:x0-linear}
\mathbf{x}_0 \;=\; a_s\,\mathbf{x}_s \;+\; b_s\,\psi_0,
\qquad
a_s \;=\; \frac{u_s}{\Delta_s},\quad
b_s \;=\; -\,\frac{\sigma_s}{\Delta_s}.
\end{equation}
Taking conditional expectation of \eqref{eq:x0-linear} given $(\mathbf{x}_s,s)$ and using $\psi^*=\mathbb{E}[\psi_0\mid \mathbf{x}_s,s]$, we obtain
\begin{equation}
\label{eq:unified-linear}
\mathbb{E}[\mathbf{x}_0\mid \mathbf{x}_s,s]
\;=\;
a_s\,\mathbf{x}_s \;+\; b_s\,\psi^*(\mathbf{x}_s,s).
\end{equation}
Thus the reconstruction rule $\hat{\mathbf{x}}_0^{(s)}:=a_s\mathbf{x}_s+b_s\psi^*(\mathbf{x}_s,s)$ exactly equals the posterior mean $\mathbb{E}[\mathbf{x}_0\mid \mathbf{x}_s,s]$; when $\psi^*\equiv\psi_0$ (ideal limit), \eqref{eq:x0-linear} is a pointwise identity. 
In particular, $\hat{\mathbf{x}}_0^{(s)}$ is an \emph{unbiased estimator} of the clean sample $\mathbf{x}_0$. 
Therefore, any admissible $\psi_0$ induces a unique pair $(a_s,b_s)$ and an equivalent reconstruction of $\mathbf{x}_0$.

\textbf{Instantiations.} We now instantiate \eqref{eq:x0-linear} for the common parameterizations by plugging the corresponding $(u_s,v_s)$:

\begin{enumerate}
\item \textbf{$\boldsymbol{\epsilon}$-prediction:} $\psi_0=\boldsymbol{\epsilon}$, i.e.\ $(u_s,v_s)=(1,0)$.
Then $\Delta_s=\alpha_s$ and
\[
a_s=\tfrac{1}{\alpha_s},\quad b_s=-\,\tfrac{\sigma_s}{\alpha_s},
\qquad
\Rightarrow\quad
\hat{\mathbf{x}}_{0}^{(s)}=\tfrac{1}{\alpha_s}\,\mathbf{x}_s-\tfrac{\sigma_s}{\alpha_s}\,\psi^*(\mathbf{x}_s,s).
\]

\item \textbf{$\mathbf{x}_0$-prediction:} $\psi_0=\mathbf{x}_0$, i.e.\ $(u_s,v_s)=(0,1)$.
Then $\Delta_s=-\sigma_s$ and
\[
a_s=0,\quad b_s=1,
\qquad
\Rightarrow\quad
\hat{\mathbf{x}}_{0}^{(s)}=\psi^*(\mathbf{x}_s,s).
\]

\item \textbf{$v$-prediction:} $\psi_0=\alpha_s\boldsymbol{\epsilon}-\sigma_s\mathbf{x}_0$, i.e.\ $(u_s,v_s)=(\alpha_s,-\sigma_s)$.
Then $\Delta_s=\alpha_s^2+\sigma_s^2$ and
\[
a_s=\frac{\alpha_s}{\alpha_s^2+\sigma_s^2},\quad
b_s=-\,\frac{\sigma_s}{\alpha_s^2+\sigma_s^2},
\qquad
\Rightarrow\quad
\hat{\mathbf{x}}_{0}^{(s)}=\frac{\alpha_s}{\alpha_s^2+\sigma_s^2}\,\mathbf{x}_s-\frac{\sigma_s}{\alpha_s^2+\sigma_s^2}\,\psi^*(\mathbf{x}_s,s).
\]
\item \textbf{Score-prediction (conditional score).}
For the conditional Gaussian 
\(q(\mathbf{x}_s\mid \mathbf{x}_0)=\mathcal{N}(\alpha_s\mathbf{x}_0,\sigma_s^2\mathbf{I})\),
\[
\nabla_{\mathbf{x}_s}\log q(\mathbf{x}_s\mid \mathbf{x}_0)
= -\,\frac{\mathbf{x}_s-\alpha_s\mathbf{x}_0}{\sigma_s^2}
= -\,\frac{1}{\sigma_s}\,\boldsymbol{\epsilon}.
\]
Hence choosing \(\psi_0=\nabla_{\mathbf{x}_s}\log q(\mathbf{x}_s\mid \mathbf{x}_0)\) corresponds to 
\((u_s,v_s)=(-1/\sigma_s,0)\).
Then \(\Delta_s=-\alpha_s/\sigma_s\) and the reconstruction coefficients are
\[
a_s=\tfrac{1}{\alpha_s},\quad b_s=\tfrac{\sigma_s^2}{\alpha_s},
\qquad
\Rightarrow\quad
\hat{\mathbf{x}}_{0}^{(s)}=\tfrac{1}{\alpha_s}\,\mathbf{x}_s+\tfrac{\sigma_s^2}{\alpha_s}\,\psi^*(\mathbf{x}_s,s).
\]

\item \textbf{Score-prediction (marginal score).}
Define the marginal distribution
\[
q_s(\mathbf{x}_s)=\int q(\mathbf{x}_s\mid \mathbf{x}_0)\,p_{\rm data}(\mathbf{x}_0)\,d\mathbf{x}_0.
\]
By the Gaussian score identity,
\[
\nabla_{\mathbf{x}_s}\log q_s(\mathbf{x}_s)
= \mathbb{E}_{\mathbf{x}_0\mid \mathbf{x}_s}\!\big[\nabla_{\mathbf{x}_s}\log q(\mathbf{x}_s\mid \mathbf{x}_0)\big]
= -\,\frac{\mathbf{x}_s-\alpha_s\,\mathbb{E}[\mathbf{x}_0\mid \mathbf{x}_s]}{\sigma_s^2}.
\]
Rearranging yields the posterior mean in closed form:
\begin{equation}
\label{eq:post-mean-from-marginal-score}
\mathbb{E}[\mathbf{x}_0\mid \mathbf{x}_s]
= \frac{1}{\alpha_s}\,\mathbf{x}_s \;+\; \frac{\sigma_s^2}{\alpha_s}\,\nabla_{\mathbf{x}_s}\log q_s(\mathbf{x}_s).
\end{equation}
If the network is trained by DSM with squared loss so that
\(\psi^*(\mathbf{x}_s,s)=\nabla_{\mathbf{x}_s}\log q_s(\mathbf{x}_s)\),
then \eqref{eq:post-mean-from-marginal-score} gives the same reconstruction rule as in the conditional case:
\[
\hat{\mathbf{x}}_{0}^{(s)}
= \tfrac{1}{\alpha_s}\,\mathbf{x}_s + \tfrac{\sigma_s^2}{\alpha_s}\,\psi^*(\mathbf{x}_s,s),
\]
with coefficients \((a_s,b_s)=(1/\alpha_s,\;\sigma_s^2/\alpha_s)\).

\textbf{Remark.} The coefficients depend only on the forward schedule \((\alpha_s,\sigma_s)\); 
the data distribution \(p_{\rm data}\) appears solely through the value of the marginal score 
\(\nabla_{\mathbf{x}_s}\log q_s(\mathbf{x}_s)\) that the network estimates.

\item \textbf{Flow matching.}
For the linear path $\mathbf{x}_s=(1-s)\mathbf{x}_0+s\,\boldsymbol{\epsilon}$ we have $\alpha_s=1-s,\ \sigma_s=s$.
A common target is $\psi_0=\boldsymbol{\epsilon}-\mathbf{x}_0$, i.e.\ $(u_s,v_s)=(1,-1)$, so $\Delta_s=\alpha_s+\sigma_s=1$ and
\[
a_s=1,\quad b_s=-\,\sigma_s=-\,s,
\qquad
\Rightarrow\quad
\hat{\mathbf{x}}_{0}^{(s)}=\mathbf{x}_s-s\,\psi^*(\mathbf{x}_s,s).
\]
Other normalizations (e.g.\ scaling $\psi_0$ by $(1-s)$) yield the corresponding $(a_s,b_s)$ via \eqref{eq:x0-linear} with the modified $(u_s,v_s)$.
\end{enumerate}

In all cases, \eqref{eq:unified-linear} shows that training with any admissible $\psi_0$ recovers the same posterior mean of $\mathbf{x}_0$ from $(\mathbf{x}_s,s)$ up to a deterministic linear readout $(a_s,b_s)$ determined solely by $(\alpha_s,\sigma_s)$ and the chosen $(u_s,v_s)$.  
Hence the parameterizations are equivalent in expressive power: they differ only in the regression target and in the reconstruction coefficients $(a_s,b_s)$.
\end{proof}

\begin{table}[t]
\centering
\caption{Unified view of common parameterizations. Each training target $\psi_0$ corresponds to a conditional regression, and the clean data $\mathbf{x}_0$ is exactly reconstructed via \eqref{eq:unified-linear}.}

\begin{tabular}{lll}
\toprule
Prediction mode & Target $\psi_0(\mathbf{x}_0,\epsilon,s)$ & Reconstruction $\hat{\mathbf{x}}_0^{(s)}$ \\
\midrule
$\epsilon$-prediction 
& $\epsilon$ 
& $\hat{\mathbf{x}}_0^{(s)} = \tfrac{1}{\alpha_s}\mathbf{x}_s - \tfrac{\sigma_s}{\alpha_s}\,\psi^*(\mathbf{x}_s,s)$ \\
$\mathbf{x_0}$-prediction 
& $\mathbf{x}_0$ 
& $\hat{\mathbf{x}}_0^{(s)} = \psi^*(\mathbf{x}_s,s)$ \\
$v$-prediction 
& $v = \alpha_s\epsilon - \sigma_s \mathbf{x}_0$ 
& $\hat{\mathbf{x}}_0^{(s)} = \tfrac{\alpha_s}{\alpha_s^2+\sigma_s^2}\mathbf{x}_s - \tfrac{\sigma_s}{\alpha_s^2+\sigma_s^2}\,\psi^*(\mathbf{x}_s,s)$ \\
$s$-prediction 
& $\nabla_{\mathbf{x}_s}\log q_s(\mathbf{x}_s) = -\tfrac{1}{\sigma_s}(\mathbf{x}_s-\tfrac{\alpha_s}{\alpha_s^2+\sigma_s^2}\mathbf{x}_0)$ 
& $\hat{\mathbf{x}}_0^{(s)} = \tfrac{1}{\alpha_s}\mathbf{x}_s + \tfrac{\sigma_s^2}{\alpha_s}\,\psi^*(\mathbf{x}_s,s)$ \\
Flow matching 
& $\epsilon-\mathbf{x}_0$ 
& $\hat{\mathbf{x}}_0^{(s)} = \mathbf{x}_s -s\,\psi^*(\mathbf{x}_s,s)$ \\
\bottomrule
\end{tabular}
\label{tab:parameterization_main}
\end{table}

\newpage
\subsection{Why Second-Order Differences are Optimal}
\label{sec:why_second_order}
We now provide a mathematical justification, entirely driven by the setting of \cref{sec:theorem} and \cref{sec:math_foundation}, for why second-order differences are (i) \emph{necessary and advantageous}, (ii) \emph{sufficient and information-theoretically optimal}, and (iii) why \emph{higher-order schemes degrade in practice}. Throughout, $t$ indexes discrete steps, while $s\in[0,1]$ denotes continuous time.

\subsubsection{Why Second-Order Differences Are Good}

\textbf{Setting.}
Let $(\Omega,\mathcal F,\Pr)$ be a probability space.
A data sample is $\mathbf \mathbf{x_0}\in L^2(\Omega;\mathbb R^d)$ with distribution $p_{\rm data}$.
Let $\boldsymbol{\epsilon}\sim\mathcal N(\mathbf 0,\mathbf I_d)$ be independent of $\mathbf \mathbf{x_0}$.
Fix $C^2$ schedules $\alpha_s,\sigma_s:[0,1]\to\mathbb R$ and define the forward latent
\[
\mathbf \mathbf{x_s}=\alpha_s\mathbf \mathbf{x_0}+\sigma_s\boldsymbol{\epsilon},\qquad s\in[0,1].
\]
Let $\psi_\theta:\mathbb R^d\times[0,1]\to\mathbb R^d$ be a trained network, twice continuously differentiable in both arguments and of at most linear growth so that $\sup_{s\in[0,1]}\mathbb E\|\psi_\theta(\mathbf \mathbf{x_s},s)\|^2<\infty$. For discrete steps $t\in\mathbb Z$ we write $\psi_t:=\psi_\theta(\mathbf x_{s_t},s_t)$ with a locally uniform grid $s_{t\pm1}=s_t\pm\Delta$.

\begin{theorem}[Signal–noise decomposition and invariance]\label{thm:decomp}
There exists a unique $\phi\in C^2([0,1];\mathbb R^d)$ and a zero-mean perturbation $\eta_t$ with $\sup_t\mathbb E\|\eta_t\|^2<\infty$ such that
\[
\psi_t=\phi(s_t)+\eta_t,\qquad \mathbb E[\eta_t\mid s_t]=0.
\]

Hence the statement holds uniformly for the usual parameterizations ($\epsilon$, $\mathbf{x_0}$, $v$, score, and flow), which are related by deterministic affine readouts in $s$.
\end{theorem}

\begin{proof}
Let 
$$\mathcal G:=\{\varphi(s_t):\ \varphi:[0,1]\to\mathbb R^d,\ \mathbb E\|\varphi(s_t)\|^2<\infty\}\subset L^2(\Omega;\mathbb R^d)
$$
It is a closed subspace. The orthogonal projection of $\psi_t$ onto $\mathcal G$ is $\phi(s_t):=\mathbb E[\psi_t\mid s_t]$, and $\eta_t:=\psi_t-\phi(s_t)$ satisfies $\mathbb E[\eta_t\mid s_t]=0$. 

Uniqueness follows from the uniqueness of Hilbert projections. 

Write \(F(s,\mathbf{x}_0,\boldsymbol{\epsilon}):=\psi_\theta(\alpha_s\mathbf{x}_0+\sigma_s\boldsymbol{\epsilon},s)\),
so \(\phi(s)=\mathbb E[F(s,\mathbf{x}_0,\boldsymbol{\epsilon})]\).
Since \(\psi_\theta\in C^2\) and \(\alpha_s,\sigma_s\in C^2\), chain rule yields
\[
\partial_s F
=\partial_s\psi_\theta(\mathbf{x}_s,s)+\nabla_x\psi_\theta(\mathbf{x}_s,s)\,(\alpha'_s\mathbf{x}_0+\sigma'_s\boldsymbol{\epsilon}),
\]
\[
\partial_s^2 F
=\partial_s^2\psi_\theta
+2\,\partial_s\nabla_x\psi_\theta\,(\alpha'_s\mathbf{x}_0+\sigma'_s\boldsymbol{\epsilon})
+\nabla_x\psi_\theta\,(\alpha''_s\mathbf{x}_0+\sigma''_s\boldsymbol{\epsilon})
+(\alpha'_s\mathbf{x}_0+\sigma'_s\boldsymbol{\epsilon})^\top\nabla_x^2\psi_\theta\,(\alpha'_s\mathbf{x}_0+\sigma'_s\boldsymbol{\epsilon}),
\]
all evaluated at \((\mathbf{x}_s,s)\).
By the polynomial growth assumption and \(\mathbf{x}_0,\boldsymbol{\epsilon}\in L^2\),
\(\partial_s F\) and \(\partial_s^2 F\) are dominated by integrable envelopes.
Thus, by dominated convergence (interchange of limit and expectation),
\[
\phi'(s)=\mathbb E[\partial_s F],\qquad \phi''(s)=\mathbb E[\partial_s^2 F],
\]
so \(\phi\in C^2([0,1])\).

Bounded variance follows from $\mathbb E\|\eta_t\|^2\le 2\sup_s\mathbb E\|\psi_\theta(\mathbf \mathbf{x_s},s)\|^2+2\sup_s\|\phi(s)\|^2<\infty$.
\end{proof}

\medskip
Theorem~\ref{thm:decomp} justifies the local model $\psi_u=\phi(u)+\eta_u$ with a smooth deterministic trend $\phi$ and a zero-mean perturbation $\eta_u$. We now study the task of \emph{reconstructing a skipped state} $\psi_{t-\Delta}$ from the two most recent computed states $\{\psi_t,\psi_{t+\Delta}\}$.

\begin{theorem}[Second-order backward extrapolation is BLUE and second-order accurate]\label{thm:blue}
Assume the decomposition in Theorem~\ref{thm:decomp} and a locally uniform grid $s_{t\pm1}=s_t\pm\Delta$. Consider linear estimators $\widehat\psi_{t-\Delta}=a\,\psi_t+b\,\psi_{t+\Delta}$ that are unbiased for all affine trends $\phi(u)=\beta_0+\beta_1 u$. Then $a=2$, $b=-1$ is the unique unbiased choice, and under homoscedastic uncorrelated perturbations $\mathrm{Var}([\eta_t,\eta_{t+\Delta}]^\top)=\sigma^2 I_2$ it minimizes the variance among all unbiased linear estimators. Moreover,
\[
\mathbb{E}[\psi_{t-\Delta}-(2\psi_t-\psi_{t+\Delta})]=\Delta^2\,\phi''(s_t)+o(\Delta^2),
\]
so the bias is $O(\Delta^2)$ for $\phi\in C^2$.
\end{theorem}

\begin{proof}
Write $y_t:=\psi_t$, $y_{t+\Delta}:=\psi_{t+\Delta}$ and stack $y=\begin{bmatrix}y_t\\ y_{t+\Delta}\end{bmatrix}$. Under an affine trend, $y=X\beta+\eta$ with
\[
X=\begin{bmatrix} 1 & s_t\\ 1 & s_t+\Delta\end{bmatrix},\quad
\beta=\begin{bmatrix}\beta_0\\ \beta_1\end{bmatrix},\quad
\eta=\begin{bmatrix}\eta_t\\ \eta_{t+\Delta}\end{bmatrix}.
\]
The target is $\theta:=\phi(s_t-\Delta)=c^\top\beta$ with $c=\begin{bmatrix}1\\ s_t-\Delta\end{bmatrix}$. A linear estimator $w^\top y$ is unbiased for all affine $\phi$ iff $w^\top X=c^\top$, i.e., $X^\top w=c$. Solving yields $w=(2,-1)^\top$ and hence $\widehat\psi_{t-\Delta}=2\psi_t-\psi_{t+\Delta}$.

For variance optimality with $\mathrm{Var}(\eta)=\sigma^2 I_2$, Gauss–Markov gives
\[
w^\star=\arg\min_{w:\,X^\top w=c}\ \mathrm{Var}(w^\top y)=\arg\min \sigma^2\|w\|_2^2\quad\Rightarrow\quad
w^\star=X(X^\top X)^{-1}c=(2,-1)^\top.
\]
Thus $2\psi_t-\psi_{t+\Delta}$ is the BLUE.

For bias, expand $\phi$ at $s_t$:
\[
\phi(s_t\pm\Delta)=\phi(s_t)\pm\Delta\,\phi'(s_t)+\tfrac{\Delta^2}{2}\phi''(s_t)+o(\Delta^2).
\]
Hence
\[
\phi(s_t-\Delta)-\bigl(2\phi(s_t)-\phi(s_t+\Delta)\bigr)=\Delta^2\,\phi''(s_t)+o(\Delta^2).
\]
Adding the zero-mean perturbations on both sides preserves the expansion in mean, proving the stated local truncation error.
\end{proof}

\begin{remark}[On conditioning in Theorem~A.4: $\phi(s)$ is a population-level trend]
Let $\mathbf{x_s}=\alpha_s \mathbf{x_0}+\sigma_s\boldsymbol{\epsilon}$ with $\mathbf{x_0}\sim p_{\rm data}$, $\boldsymbol{\epsilon}\sim\mathcal N(0,I)$, and assume $s$ is independent of $(\mathbf{x_0},\boldsymbol{\epsilon})$. 
For a target $\psi_t=\psi_0(\mathbf{x_0},\boldsymbol{\epsilon},s_t)$, Theorem~A.4 defines the signal part by
\[
\phi(s)\ :=\ \mathbb E[\psi_t\,|\,s_t=s].
\]
Importantly, the conditioning is \emph{only} on the time index $s$ (not on the realization $\mathbf{x_s}$), hence $\phi$ is a deterministic function of $s$.
By the tower property,
\[
\phi(s)
= \mathbb E\!\big[\psi_0(\mathbf{\mathbf{x_0}},\boldsymbol{\epsilon},s)\,\big|\,s\big]
= \mathbb E\!\big[\,\mathbb E\!\big[\psi_0(\mathbf{x_0},\boldsymbol{\epsilon},s)\,\big|\,\mathbf{x_s},s\big]\,\big|\,s\big],
\]
so $\phi(s)$ is obtained by first taking the posterior mean given $(\mathbf{x_s},s)$ and then marginalizing over $\mathbf{x_s}\sim q_s$.
This is fundamentally different from \emph{posterior reconstruction} (e.g., $\hat{ \mathbf{x}}_0(\mathbf{x_s},s):=\mathbb E[\mathbf{x_0}\,|\,\mathbf{x_s},s]$), which depends on the particular observation $\mathbf{x_s}$.

For the usual parameterizations we get closed forms. 
We use that $\mathbb E[\boldsymbol{\epsilon}\,|\,s]=\mathbb E[\boldsymbol{\epsilon}]=0$ and $\mathbb E[\mathbf{x_0}\,|\,s]=\mathbb E[\mathbf{x_0}]$ by independence.
\begin{itemize}
\item \textbf{$\epsilon$-prediction:} $\psi_0(\mathbf{x_0},\boldsymbol{\epsilon},s)=\boldsymbol{\epsilon}$. Then $\phi(s)=\mathbb E[\boldsymbol{\epsilon}\,|\,s]=0$.
\item \textbf{$\mathbf{x_0}$-prediction:} $\psi_0(\mathbf{x_0},\boldsymbol{\epsilon},s)=\mathbf{x_0}$. Then $\phi(s)=\mathbb E[\mathbf{x_0}\,|\,s]=\mathbb E[\mathbf{x_0}]$.
\item \textbf{$v$-prediction:} $\psi_0(\mathbf{x_0},\boldsymbol{\epsilon},s)=\alpha_s\boldsymbol{\epsilon}-\sigma_s \mathbf{x_0}$. Hence 
$\phi(s)=\alpha_s\mathbb E[\boldsymbol{\epsilon}\,|\,s]-\sigma_s\mathbb E[\mathbf{x_0}\,|\,s]=-\sigma_s\,\mathbb E[\mathbf{x_0}]$.
\item \textbf{Score-prediction:} $\psi_0(\mathbf{x_0},\boldsymbol{\epsilon},s)=\nabla_{\mathbf{x_s}}\log q_s(\mathbf{x_s})$, so
$\phi(s)=\mathbb E_{\mathbf{x_s}\sim q_s}[\nabla\log q_s(\mathbf{x_s})]=0$.
\item \textbf{Flow-prediction:} $\psi_0(\mathbf{x_0},\boldsymbol{\epsilon},s)=\frac{d}{ds}\mathbf{x_s}=\alpha_s' \mathbf{x_0}+\sigma_s'\boldsymbol{\epsilon}$, so
$\phi(s)=\alpha_s'\mathbb E[\mathbf{x_0}\,|\,s]+\sigma_s'\mathbb E[\boldsymbol{\epsilon}\,|\,s]=\alpha_s'\mathbb E[\mathbf{x_0}]$.
\end{itemize}

Equivalently, whenever a parameterization is an affine readout 
$\psi_0(\mathbf{x_0},\boldsymbol{\epsilon},s)=a(s)\,\boldsymbol{\epsilon}+b(s)\,\mathbf{x_0}+c(s)$, we have the general identity
\[
\phi(s)=b(s)\,\mathbb E[\mathbf{x_0}]+c(s),
\]
since $\mathbb E[\boldsymbol{\epsilon}]=0$ and $s\!\perp\!(\mathbf{x_0},\boldsymbol{\epsilon})$.

\medskip
\begin{table}[t]
\centering
\caption{Summary of different parameterizations.}
\label{tab:parameterization}
\begin{tabular}{lc}
\hline
Parameterization & $\phi(s) = \mathbb{E}[\psi_t \mid s_t=s]$ \\
\hline
$\epsilon$-prediction   & $0$ \\
$\mathbf{x_0}$-prediction        & $\mathbb{E}[\mathbf{x_0}]$ \\
$v$-prediction          & $-\,\sigma_s\,\mathbb{E}[\mathbf{x_0}]$ \\
Score-prediction        & $0$ \\
Flow-prediction         & $\alpha_s'\,\mathbb{E}[\mathbf{x_0}]$ \\
\hline
\end{tabular}
\end{table}

Thus assuming that $\phi(s)$ is an afine trends is reasonable.
\end{remark}

\begin{proposition}[No two-point linear estimator beats the $O(\Delta^2)$ order]\label{prop:order}
Let $\widehat\psi_{t-\Delta}=a\,\psi_t+b\,\psi_{t+\Delta}$ be any estimator that is unbiased for all affine $\phi$. Then for every such choice,
\[
\psi_{t-\Delta}-\widehat\psi_{t-\Delta}=K(a,b)\,\Delta^2\,\phi''(s_t)+o(\Delta^2)
\quad\text{with }K(a,b)\neq 0,
\]
so the order $O(\Delta^2)$ cannot be improved using only $\{\psi_t,\psi_{t+\Delta}\}$.
\end{proposition}

\begin{proof}
Unbiasedness for all affine $\phi$ enforces the constraints $a+b=1$ and $a s_t+b(s_t+\Delta)=s_t-\Delta$, which have the unique solution $a=2$, $b=-1$. For any $C^2$ trend, Taylor's theorem with remainder gives the error coefficient $K(2,-1)=1$. If one relaxed unbiasedness, matching constants and linears is still necessary to avoid $O(1)$ or $O(\Delta)$ bias uniformly in $\phi$; with only two samples, the highest degree one can reproduce is $1$, hence the Peano kernel argument yields an $O(\Delta^2)$ remainder with a nonzero coefficient for some $\phi''$.
\end{proof}

\begin{corollary}[Curvature observability]\label{cor:curv}
The second difference $\Delta^{(2)}\psi_t:=\psi_{t+\Delta}-2\psi_t+\psi_{t-\Delta}$ cancels any affine trend and isolates curvature:
\[
\Delta^{(2)}\phi(s_t)=\Delta^2\,\phi''(s_t)+o(\Delta^2).
\]
Thus the BLUE extrapolator $2\psi_t-\psi_{t+\Delta}$ explicitly exploits $\Delta^{(2)}$ to reconstruct the skipped state $\psi_{t-\Delta}$ while being insensitive to large affine drifts in $\phi$.
\end{corollary}

\begin{remark}[Correlated perturbations and generalized least squares]
If $\mathrm{Var}(\eta)=\Sigma\succ0$ (not necessarily diagonal), the BLUE weights become
\[
(w^\star)^\top=c^\top\bigl(X^\top\Sigma^{-1}X\bigr)^{-1}X^\top\Sigma^{-1}.
\]
When $\Sigma=\sigma^2 I_2$ this reduces to $(2,-1)$.
\end{remark}

\textbf{Takeaway.}
Under the mild, parameterization-invariant decomposition of Theorem~\ref{thm:decomp}, the backward second-order rule
\[
\widehat\psi_{t-\Delta}=2\psi_t-\psi_{t+\Delta}
\]
is simultaneously (i) unbiased for all affine trends, (ii) variance-optimal among all linear unbiased two-point estimators, (iii) second-order accurate with bias $O(\Delta^2)$, and (iv) curvature-aware through $\Delta^{(2)}$. None of these guarantees is achievable with zeroth- or first-order reuse from $\{\psi_t,\psi_{t+\Delta}\}$ alone.

From the next section onward, we omit the perturbation terms~$\eta$ and focus solely on the underlying trend $\phi$, as the $(2,-1)$ rule has already been shown to be the unique BLUE. Any additive zero-mean noise merely inflates the estimation variance but cannot reduce the inherent bias floor.

\subsubsection{Why Second-Order is Sufficient}
\label{sec:second-order-sufficient}

\paragraph{Setup.}
Let the unknown target function $\phi:[0,1]\to\mathbb R^d$ belong to the class
\[
\mathcal F(M_2)\;=\;\Big\{\phi\in C^2([0,1];\mathbb R^d):\ \sup_{s\in[0,1]}\|\phi''(s)\|\le M_2\Big\}.
\]
At discrete step $t$, we only have access to the model evaluations
\[
\psi_t=\phi(s_t),\qquad \psi_{t+\Delta}=\phi(s_{t+\Delta}),
\]
with $s_{t\pm \Delta}=s_t\pm \Delta$.
The goal is to estimate $\phi(s_{t-\Delta})$ based on these observations.
Denote a generic estimator by
\[
\widehat\psi_{t-\Delta}\;=\;\mathsf{Alg}\big(\psi_t,\psi_{t+\Delta}\big).
\]

We show that under only $C^2$ regularity, every estimator incurs worst-case bias of order $\Delta^2$, while the standard second-order extrapolation
\begin{equation}\label{eq:backward-extrap}
\widehat\psi_{t-\Delta}^{(2)} \;=\; 2\,\psi_t - \psi_{t+\Delta}
\end{equation}
achieves this rate. Thus, second-order is minimax optimal.

\paragraph{Lower Bound via Two-Point Method.}

\begin{theorem}[Two-point lower bound]\label{thm:two-point-lb}
For any estimator $\widehat\psi_{t-\Delta}$ depending only on $\{\psi_t,\psi_{t+\Delta}\}$, there exist $\phi_\pm\in\mathcal F(M_2)$ such that
\[
\phi_\pm(s_t)=\phi_\pm(s_{t+\Delta})=0,\qquad 
\big\|\phi_+(s_{t-\Delta})-\phi_-(s_{t-\Delta})\big\|\;\ge\;c\,M_2\,\Delta^2,
\]
for some absolute constant $c>0$. Consequently,
\[
\inf_{\widehat\psi_{t-\Delta}}\ \sup_{\phi\in\mathcal F(M_2)}
\big\|\widehat\psi_{t-\Delta}-\phi(s_{t-\Delta})\big\|
\ =\ \Omega(M_2\Delta^2).
\]
\end{theorem}

\begin{proof}
Let $u=(s-s_t)/\Delta$ and consider the quadratic Lagrange basis polynomial
\[
p(u) = \tfrac12 u(u-1),\qquad p(0)=p(1)=0,\; p(-1)=1,\; p''(u)\equiv 1.
\]
Define $g(s) = M_2\Delta^2\,p\!\big((s-s_t)/\Delta\big)$ and set $\phi_\pm=\pm g$.
Then $\phi_\pm\in\mathcal F(M_2)$, agree at $s_t$ and $s_{t+\Delta}$, but differ by
\[
\|\phi_+(s_{t-\Delta})-\phi_-(s_{t-\Delta})\|
=2M_2\Delta^2.
\]
Since the data are indistinguishable, Le Cam’s two-point method \citep{lecam1973convergence} implies any estimator incurs at least half this separation on one of the two instances, yielding the stated $\Omega(M_2\Delta^2)$ bound.
\end{proof}

\paragraph{Transition.}
Theorem~1 shows that $\Delta^2$ bias is an information-theoretic lower bound. Next we show that the second-order extrapolation \eqref{eq:backward-extrap} matches this rate.

\paragraph{Upper Bound via Second-Order Extrapolation.}
\begin{theorem}[Achievability]\label{thm:achievability}
For any $\phi\in\mathcal F(M_2)$,
\[
\big\|\phi(s_{t-\Delta})-(2\phi(s_t)-\phi(s_{t+\Delta}))\big\|
\ \le\ M_2\,\Delta^2
\ =\ \mathcal O(M_2\Delta^2).
\]
\end{theorem}

\begin{proof}
By Taylor’s theorem, for some $\theta_1,\theta_2\in(0,1)$,
\[
\begin{aligned}
\phi(s_{t+\Delta})&=\phi(s_t)+\phi'(s_t)\Delta+\tfrac12\,\phi''(s_t+\theta_1\Delta)\Delta^2,\\
\phi(s_{t-\Delta})&=\phi(s_t)-\phi'(s_t)\Delta+\tfrac12\,\phi''(s_t-\theta_2\Delta)\Delta^2.
\end{aligned}
\]
Subtracting yields
\[
\phi(s_{t-\Delta})-\big(2\phi(s_t)-\phi(s_{t+\Delta})\big)
=\tfrac12\big[\phi''(s_t-\theta_2\Delta)+\phi''(s_t+\theta_1\Delta)\big]\Delta^2,
\]
whose norm is bounded by $M_2\Delta^2$.
\end{proof}

\begin{corollary}[Minimax rate]
Combining Theorems 1 and 2, the minimax rate for estimating $\phi(s_{t-\Delta})$ under $C^2$ regularity is $\Delta^2$, achieved by the second-order extrapolation \eqref{eq:backward-extrap}.
\end{corollary}

\paragraph{Variance and BLUE.}
Suppose further that observations are corrupted by i.i.d.\ noise:
\[
\psi_t=\phi(s_t)+\eta_t,\qquad 
\psi_{t+\Delta}=\phi(s_{t+\Delta})+\eta_{t+\Delta},\qquad 
\mathbb E[\eta_t]=0,\ \operatorname{Var}(\eta_t)=\sigma^2.
\]
Consider linear unbiased estimators $\widehat\psi_{t-\Delta}=a\psi_t+b\psi_{t+\Delta}$.
Unbiasedness for all affine $\phi(s)$ requires
\[
\begin{bmatrix}1&1\\ s_t&s_{t+\Delta}\end{bmatrix}
\begin{bmatrix}a\\ b\end{bmatrix}
=\begin{bmatrix}1\\ s_{t-\Delta}\end{bmatrix},
\]
whose unique solution is $(a,b)=(2,-1)$. The variance is then
\[
\operatorname{Var}(\widehat\psi_{t-\Delta})=(a^2+b^2)\sigma^2=5\sigma^2.
\]
Thus \eqref{eq:backward-extrap} is the unique best linear unbiased estimator (BLUE).

\paragraph{Complexity Perspective.}
Three points uniquely determine a quadratic interpolant, achieving order-$\Delta^2$ bias. Adding more points to fit higher-degree polynomials cannot improve the minimax rate, since functions in $\mathcal F(M_2)$ need not possess bounded higher derivatives. In fact, the Lebesgue constant
\[
\Lambda_n=\max_{s}\sum_{j=0}^n |\ell_j(s)|
\]
of polynomial interpolation typically grows with the number of nodes, degrading stability. Hence ``more points’’ only increase constants without reducing the minimax order. Moreover, since our sampling points are nearly uniform, the resulting polynomial interpolant 
is susceptible to Runge’s phenomenon~\citep{runge1901empirische} (see \cref{fig:runge}), which may further degrade stability (see \cref{thm:lebesgue} for more details).

\begin{figure}[t]
    \centering
    \includegraphics[width=0.8\linewidth]{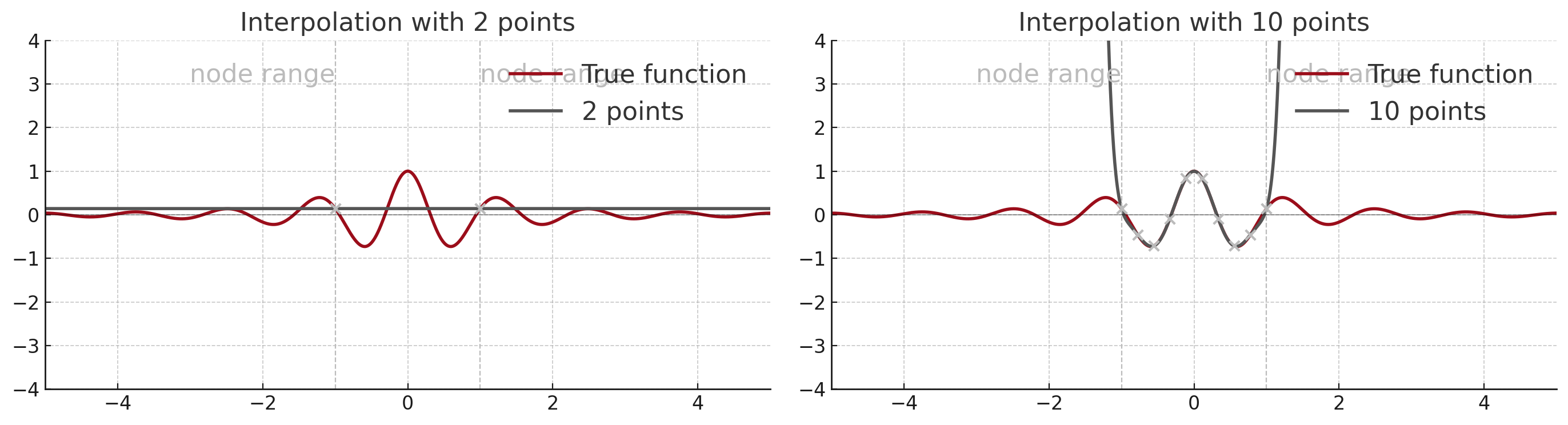}
    \caption{Illustration of the Runge phenomenon.
    Left: polynomial interpolation with only 2 nodes. 
    Right: interpolation with 10 nodes, which exhibits severe oscillations and divergence near the boundary.}
    \label{fig:runge}
\end{figure}

\paragraph{Conclusion.}
Under scarce fresh evaluations and only $C^2$ regularity (the weakest assumption justified by the forward process, cf.\ Appendix~A.1), second-order extrapolation \eqref{eq:backward-extrap} is \emph{information-theoretically optimal}: it matches the $\Delta^2$ minimax lower bound and is BLUE among linear unbiased estimators. Thus it lies on the Pareto frontier \citep{pareto1919manuale} of the bias–variance tradeoff, and additional points cannot improve the minimax rate while often worsening numerical stability.

\subsubsection{Why First-Order Reuse Is Insufficient}
\label{sec:first-order-bad}

We quantify the best possible accuracy if one only reuses a \emph{single} model evaluation (e.g., $\widehat\psi_{t-\Delta}=\psi_t$ or $\widehat\psi_{t-\Delta}=\psi_{t+\Delta}$). Let

$$
\mathcal F_1(M_1)\;=\;\Big\{\phi\in C^1([0,1];\mathrm R^d):\ \sup_{s\in[0,1]}\|\phi'(s)\|\le M_1\Big\}.
$$

Since $\phi\in C^2([0,1])$ on a compact interval implies $\phi'\in C^1$ and hence bounded, the class $\mathcal F_1(M_1)$ is compatible with the $C^2$ setting of \S\ref{sec:second-order-sufficient}.

\paragraph{Information-theoretic lower bound for one-point estimators.}
Consider any estimator that depends on a single point,

$$
\widehat\psi_{t-\Delta} \;=\; \mathsf{Alg}\big(\psi_\tau\big),\qquad \tau\in\{t,\ t{+}\Delta\},
$$

with $\psi_u=\phi(s_u)$ and $s_{t\pm\Delta}=s_t\pm\Delta$.

\begin{theorem}[One-point minimax lower bound]
\label{thm:one-point-lb}
For any measurable estimator $\widehat\psi_{t-\Delta}=\mathsf{Alg}(\psi_\tau)$ using only one of ${\psi_t,\psi_{t+\Delta}}$,

$$
\inf_{\widehat\psi_{t-\Delta}}\ \sup_{\phi\in\mathcal F_1(M_1)}
\big\|\widehat\psi_{t-\Delta}-\phi(s_{t-\Delta})\big\|
\;=\;\Omega(M_1\Delta).
$$

\end{theorem}

\begin{proof}
WLOG take $\tau=t$ (the case $\tau=t{+}\Delta$ is analogous). Define two affine trends

$$
\phi_\pm(s)\;=\;\pm M_1\,(s-s_t).
$$

Then $\phi_\pm\in\mathcal F_1(M_1)$, and they \emph{agree} at $s_t$:
$\phi_+(s_t)=\phi_-(s_t)=0$, hence the observation $\psi_t$ is identical in both cases. But at the target location,

$$
\big\|\phi_+(s_{t-\Delta})-\phi_-(s_{t-\Delta})\big\|
= \big\|{+}M_1(-\Delta)-(-M_1(-\Delta))\big\|=2M_1\Delta.
$$

Since the data are indistinguishable, Le Cam’s two-point method \citep{lecam1973convergence} implies any estimator incurs at least half this separation on one of the two instances, yielding the stated $\Omega(M_1\Delta)$ bound.
\end{proof}

\paragraph{Achievability with naive reuse.}
The trivial reuse $\widehat\psi_{t-\Delta}=\psi_t$ attains this rate.

\begin{proposition}[First-order bias upper bound]
\label{prop:first-order-ub}
For any $\phi\in\mathcal F_1(M_1)$,

$$
\big\|\phi(s_{t-\Delta})-\psi_t\big\|
=\big\|\phi(s_{t-\Delta})-\phi(s_t)\big\|
\ \le\ M_1\,\Delta.
$$

Thus the one-point minimax rate is $\Theta(M_1\Delta)$.
\end{proposition}

\begin{proof}
By the mean-value theorem,
$\phi(s_{t-\Delta})-\phi(s_t)=-\Delta\,\phi'(s_t-\theta\Delta)$ for some $\theta\in(0,1)$, so the norm is $\le M_1\Delta$.
\end{proof}

\paragraph{Consequences (vs.\ second-order).}
\begin{itemize}
\item \textbf{Bias order:} Any one-point scheme is \emph{at best} first-order accurate (bias $\Theta(\Delta)$), while the two-point second-order extrapolation $2\psi_t-\psi_{t+\Delta}$ is \emph{second-order} (bias $\Theta(\Delta^2)$), cf.\cref{sec:second-order-sufficient}.
\item \textbf{Linear-unbiased restriction:} If we additionally require linear unbiasedness for constants (natural for reuse), $\widehat\psi_{t-\Delta}=a\,\psi_t$ forces $a=1$, so

$$
\phi(s_{t-\Delta})-\widehat\psi_{t-\Delta}
= \phi(s_{t-\Delta})-\phi(s_t)
= -\Delta\,\phi'(s_t)+\tfrac{\Delta^2}{2}\phi''(s_t)+o(\Delta^2),
$$

whose leading term is generically $O(\Delta)$ and cannot be removed without using \emph{two} points to cancel the linear drift.
\end{itemize}

\paragraph{Takeaway.}
Any estimator that “just reuses” a single evaluation is \emph{information-theoretically} limited to an $O(\Delta)$ bias (Theorem~\ref{thm:one-point-lb}), which the trivial reuse $\widehat\psi_{t-\Delta}=\psi_t$ already attains (Proposition~\ref{prop:first-order-ub}). In contrast, the two-point second-order rule cancels the linear drift and reaches the $\Theta(\Delta^2)$ minimax rate under $C^2$ regularity (\cref{thm:two-point-lb} and \cref{thm:achievability}). Hence first-order reuse is \emph{necessarily suboptimal}.

\subsubsection{Why Higher-Order Extrapolation is Detrimental}
\label{sec:higher-order-bad}

We now establish, in a formal manner, why schemes of order $k>2$ are not
advantageous under the $C^2$ regularity available in diffusion ODE
sampling. Despite the apparent flexibility of higher-order stencils,
they suffer from four fundamental drawbacks: exponential noise
amplification, curvature mixing, collapse of stability domain, and
interpolation sensitivity.

\begin{theorem}[Exponential noise amplification]
\label{thm:noise}
Let $w^{(k)}\in\mathrm R^{k+1}$ be the Lagrange weights for extrapolating
$\psi_{-1}$ from $\{\psi_0,\ldots,\psi_k\}$. Then
\[
  \|w^{(k)}\|_2^2 = \binom{2k+2}{k+1}
  \;\sim\; \tfrac{4^{k+1}}{\sqrt{\pi(k+1)}},
\]
and hence the variance of the extrapolate obeys
\(
  \mathrm Var(\widehat\psi^{(k)}_{-1}) = \Omega(4^k\sigma^2).
\)
\end{theorem}
\begin{proof}
For equispaced nodes $\{0,1,\ldots,k\}$ the $j$-th Lagrange basis reads
\[
  l_j(x) = \prod_{\substack{m=0 \\ m\neq j}}^{k} \frac{x-m}{j-m}.
\]
Evaluating at $x=-1$ gives
\[
  w^{(k)}_j = l_j(-1) = \frac{\prod_{m=0,m\neq j}^{k} (-1-m)}{\prod_{m=0,m\neq j}^{k}(j-m)}.
\]
The numerator simplifies as
\[
  \prod_{m=0,m\neq j}^{k} (-1-m)
  = \frac{\prod_{m=0}^{k} (-1-m)}{-1-j}
  = \frac{(-1)^{k+1}(k+1)!}{-(j+1)}
  = (-1)^k \frac{(k+1)!}{j+1}.
\]
The denominator splits into two products:
\[
  \prod_{m=0,m\neq j}^{k}(j-m)
  = \bigg(\prod_{m=0}^{j-1}(j-m)\bigg)\bigg(\prod_{m=j+1}^{k}(j-m)\bigg)
  = j!\,(-1)^{k-j}(k-j)!.
\]
Combining yields
\[
  w^{(k)}_j = (-1)^j \frac{(k+1)!}{(j+1)!\,(k-j)!}
  = (-1)^j \binom{k+1}{j+1}.
\]

Therefore the squared $\ell_2$ norm is
\[
  \|w^{(k)}\|_2^2 = \sum_{j=0}^{k} \binom{k+1}{j+1}^2
  = \sum_{m=1}^{k+1} \binom{k+1}{m}^2.
\]
By the Chu–Vandermonde identity
$\sum_{m=0}^{n}\binom{n}{m}^2=\binom{2n}{n}$, we obtain
\[
  \|w^{(k)}\|_2^2 = \binom{2k+2}{k+1}.
\]

Finally, Stirling’s formula for the central binomial coefficient gives
\[
  \binom{2k+2}{k+1} \sim \frac{4^{k+1}}{\sqrt{\pi(k+1)}},
\]
hence
\[
  \mathrm Var(\widehat\psi_{-1}^{(k)})
  = \sigma^2 \|w^{(k)}\|_2^2
  = \Omega(4^k\sigma^2),
\]
as claimed.
\end{proof}

\begin{theorem}[Interpolation sensitivity]
\label{thm:lebesgue}
Let $\Lambda_k$ be the Lebesgue constant of equispaced interpolation
using $k+1$ nodes. Then $\Lambda_k$ grows exponentially in $k$, whereas
for optimal Chebyshev nodes it grows only logarithmically. Hence,
higher-order extrapolants with equispaced stencils are exponentially
sensitive to perturbations or irregular steps.
\end{theorem}

\begin{proof}
The exponential growth already appeared in Theorem~\ref{thm:noise}:
the $\ell_2$ norm of the weights $\|w^{(k)}\|_2$ scales like $4^k$,
so any perturbation in the data is magnified accordingly.
This phenomenon is precisely quantified by the Lebesgue constant
\[
  \Lambda_k = \sup_x \sum_{j=0}^k |l_j(x)|,
\]
which measures the operator norm of the interpolation map.
Classical interpolation theory ~\citep{runge1901empirische} shows that for equispaced
nodes $\Lambda_k \sim c\,2^k$ with $c>0$. Since diffusion ODE sampling necessarily
uses uniform timesteps, we inherit the exponential sensitivity of equispaced interpolation.
\end{proof}

\paragraph{Conclusion.}
Together, these theorems show that higher-order extrapolation schemes
amplify stochastic noise by $\Omega(4^k)$, destroy curvature alignment
through alternating differences, shrink the admissible stability domain
beyond usefulness, and become exponentially sensitive to small
perturbations. Under the $C^2$ smoothness regime of diffusion processes,
the minimax bias rate is $\Omega(\Delta^2)$, already attained by
second-order schemes. Thus higher-order methods offer no bias
improvement but introduce severe variance and stability costs. In
realistic sampling budgets, second-order extrapolation is both necessary
and sufficient for robust skipping in diffusion ODEs.

\subsection{Why is second-order differencing the only possibility under ambitious skipping?}

We now make precise the consequence of enforcing uniform skipping with a speed-up factor of at least two. Under such uniform spacing, the pigeonhole principle implies that no two consecutive steps can both be fresh, and every three-step local window contains at most one fresh evaluation. As a result, the only minimally sufficient real-computation unit is formed by a single fresh value together with its deterministic difference against the next state, yielding a unique second-order difference rule.

\begin{theorem}[Uniform $\,\ge 2\times$ skipping forbids consecutive fresh steps and yields a unique minimal tuple]
\label{thm:uniform-no-consecutive}
Let steps be $1,2,\ldots,n$ on a uniform grid. Each step is either produced by a fresh network evaluation (“fresh”) or by purely algebraic reuse of already computed values. 
Assume the skipping pattern is \emph{uniform}: there exists an integer $r\ge 2$ and a residue class $c\in\{1,\ldots,r\}$ such that the fresh steps are exactly those indices
\[
\mathcal F\;=\;\{\,i\in\{1,\ldots,n\}: i\equiv c \!\!\!\pmod r\,\},
\]
and every other step is obtained by algebraic reuse (no extra network calls). 
Then:
\begin{enumerate}
\item No two consecutive steps can both be fresh.
\item In any local window $\{t-1,t,t+1\}$, there is at most one fresh step.
\item If $\psi_{t-1}$ is skipped (to be reconstructed by reuse), the only minimally viable real-computation tuple is
\[
\big(\,\psi_t,\ \Delta\psi_t\,\big),\qquad \Delta\psi_t:=\psi_t-\psi_{t+1},
\]
in the sense that $\psi_t$ is the unique fresh evaluation in the window and $\Delta\psi_t$ is a deterministic difference (trend) computable without an additional network call. 
No strictly smaller tuple can determine both level and local trend, and any tuple containing two fresh entries contradicts the uniform $r\!\ge\!2$ spacing.
\end{enumerate}
\end{theorem}

\begin{proof}
By uniformity, any two fresh indices differ by a multiple of $r\ge 2$. Hence the gap between consecutive fresh steps is at least $2$, so no two adjacent indices can both be fresh, proving (1). Consequently, any three-consecutive-step window $\{t-1,t,t+1\}$ contains at most one fresh index, proving (2).

For (3), consider a window $\{t-1,t,t+1\}$ in which $\psi_{t-1}$ is skipped. By (2), at most one of these is fresh. If $t-1$ were fresh, no reconstruction would be needed; thus (w.l.o.g.) $t$ is the unique fresh index. Any additional quantity used to infer $\psi_{t-1}$ must be obtained without further network calls. Since $t+1$ cannot be fresh when $t$ is fresh under the uniform $r\!\ge\!2$ spacing, reusing $\psi_{t+1}$ is deterministic. The first difference $\Delta\psi_t:=\psi_t-\psi_{t+1}$ then supplies independent information about the local trend around $t$.

Minimality follows from identifiability: a single fresh value $\psi_t$ fixes only the local “level.” Without at least one independent trend descriptor (e.g., $\Delta\psi_t$), $\psi_{t-1}$ is not determined in general (e.g., under an affine local model one needs both level and slope). Thus any tuple strictly smaller than $(\psi_t,\Delta\psi_t)$ is insufficient. Conversely, any tuple with two fresh entries violates the uniform gap $r\ge 2$. Hence $(\psi_t,\Delta\psi_t)$ is the unique minimally sufficient real-computation unit under uniform $\ge 2\times$ skipping.
\end{proof}

\subsection{Multi-step error under reuse and two-point extrapolation}
\label{app:multistep-error}

We study the bias and variance of three strategies for jumping left by $j$ grid points from an anchor at $s_t=t\Delta$ when only the two most recent observations $\{\psi_t,\psi_{t+1}\}$ are available. 
Assume a $d$-dimensional smooth ground-truth trajectory $\psi^*:[0,1]\to\mathrm R^d$ with $\psi^*\in C^4([0,1])$, and noisy observations
\[
\psi_{u}=\psi^*(s_{u})+\eta_{u},\qquad \mathrm E[\eta_u]=0,\quad \mathrm Var(\eta_u)=\sigma^2 I_d,\quad \eta_u\ \text{uncorrelated across }u.
\]
For any estimator $\widehat\psi_{t-j}$ of $\psi^*_{t-j}:=\psi^*(s_{t-j})$, define its mean-bias and variance by
\[
\mathrm{Bias}(\widehat\psi_{t-j}):=\mathrm E[\widehat\psi_{t-j}]-\psi^*_{t-j},\qquad 
\mathrm{Var}(\widehat\psi_{t-j}):=\mathrm Var(\widehat\psi_{t-j}),
\]
and the mean-squared error $\mathrm{MSE}=\|\mathrm{Bias}\|_2^2+\mathrm{tr}\,\mathrm{Var}$.
Below, all big-$O$ terms are uniform in $t$ and $j$ as $\Delta\to0$; derivatives of $\psi^*$ are evaluated at $s_t$ unless stated otherwise.

\medskip
\noindent\textbf{Three strategies.}
(A) \emph{Two-point linear extrapolation at a single anchor}: for any $j\ge1$,
\[
\widehat\psi^{A}_{t-j}:=(j{+}1)\psi_{t}-j\psi_{t+1}.
\]
(B) \emph{Interval reuse}: first compute $\widehat\psi_{t-1}=2\psi_t-\psi_{t+1}$, and then reuse every other step to the left, i.e.\ $\widehat\psi^{B}_{s}=\widehat\psi^{B}_{s+2}$ for $s\le t-2$. Equivalently,
\[
\widehat\psi^{B}_{t-j}=
\begin{cases}
\psi_t, & j\ \text{even},\\
2\psi_t-\psi_{t+1}, & j\ \text{odd}.
\end{cases}
\]
(C) \emph{Pure reuse}: ignore $\psi_{t+1}$ and set $\widehat\psi^{C}_{t-j}:=\psi_t$ for all $j\ge1$.

\begin{lemma}[Closed form and equivalence]\label{lem:lin-closed}
Strategy {\rm(A)} is the unique sequence generated by the second-order recurrence
$\widehat\psi_{t-(k+1)}=2\widehat\psi_{t-k}-\widehat\psi_{t-(k-1)}$ with initial conditions $\widehat\psi_{t}=\psi_t$ and $\widehat\psi_{t-1}=2\psi_t-\psi_{t+1}$. In particular,
$\widehat\psi_{t-j}^{A}=(j{+}1)\psi_t-j\psi_{t+1}$ for all $j\ge1$.
\end{lemma}

\begin{proof}
Solve the linear homogeneous recurrence $x_{k+1}-2x_k+x_{k-1}=0$, whose general solution is $x_k=\alpha+\beta k$. With $\mathbf{x_0}=\psi_t$ and $x_1=2\psi_t-\psi_{t+1}$ one obtains $x_k=(k+1)\psi_t-k\psi_{t+1}$. Uniqueness follows from linearity.
\end{proof}

\begin{theorem}[Bias of the three strategies]\label{thm:bias}
Let $M_r:=\sup_{s\in[0,1]}\|\,\psi^{*(r)}(s)\,\|$ for $r=1,2,3,4$. Then, uniformly in $t$ and $j$:

\noindent{\rm(A)} For two-point linear extrapolation,
\[
\|\mathrm{Bias}(\widehat\psi^{A}_{t-j})\|
\ \le\ \tfrac12\,j(j{+}1)\,M_2\,\Delta^2\ =\ O(j^2\Delta^2).
\]
Moreover, the exact Taylor expansion is
\[
\mathrm{Bias}(\widehat\psi^{A}_{t-j})
= -\tfrac12\,j(j{+}1)\,\psi^{*\,\prime\prime}_t\Delta^2
+ \tfrac16\,(j^3{-}j)\,\psi^{*\, (3)}_t\Delta^3
- \tfrac1{24}\,(j^4{+}j)\,\psi^{*\, (4)}_t\Delta^4
+ O(\Delta^5).
\]

\noindent{\rm(B)} For interval reuse, letting $j=2k$ or $j=2k{+}1$,
\[
\begin{aligned}
\|\mathrm{Bias}(\widehat\psi^{B}_{t-2k})\|
&\le\ 2k\,M_1\,\Delta+\tfrac12(2k)^2M_2\Delta^2+\tfrac16(2k)^3M_3\Delta^3+O(\Delta^4)
=O(j\Delta),\\
\|\mathrm{Bias}(\widehat\psi^{B}_{t-(2k+1)})\|
&\le\ 2k\,M_1\,\Delta+(2k^2{+}2k{+}1)M_2\Delta^2+O(\Delta^3)=O(j\Delta).
\end{aligned}
\]
The leading-order expansions are, respectively,
\begin{align*}
   \mathrm{Bias}(\widehat\psi^{B}_{t-2k})
&= 2k\,\psi^{*\,\prime}_t\Delta - 2k^2\,\tfrac{\psi^{*\,\prime\prime}_t}{1}\Delta^2 + O(\Delta^3),\\
\mathrm{Bias}(\widehat\psi^{B}_{t-(2k+1)})
&= 2k\,\psi^{*\,\prime}_t\Delta -(2k^2{+}2k{+}1)\psi^{*\,\prime\prime}_t\Delta^2 + O(\Delta^3).
\end{align*}

\noindent{\rm(C)} For pure reuse,
\[
\|\mathrm{Bias}(\widehat\psi^{C}_{t-j})\|
\le j\,M_1\,\Delta+\tfrac12 j^2 M_2\Delta^2+\tfrac16 j^3 M_3\Delta^3+O(\Delta^4)
=O(j\Delta),
\]
with leading expansion
\(
\mathrm{Bias}(\widehat\psi^{C}_{t-j})=j\,\psi^{*\,\prime}_t\Delta-\tfrac12 j^2\psi^{*\,\prime\prime}_t\Delta^2+O(\Delta^3).
\)
\end{theorem}
\begin{proof}
Fix a uniform grid with step size $\Delta>0$ so that $s_{t+1}=s_t+\Delta$ and $s_{t-j}=s_t-j\Delta$. 
Let $\psi^*:[0,1]\to\mathbb R^d$ be $C^4$ and denote derivatives at $s_t$ by
\[
\psi:=\psi^*(s_t),\qquad 
\psi' := \psi^{*(1)}(s_t),\quad 
\psi'' := \psi^{*(2)}(s_t),\quad 
\psi''' := \psi^{*(3)}(s_t),\quad 
\psi'''' := \psi^{*(4)}(s_t).
\]
The observations are $\psi_u=\psi^*(s_u)+\eta_u$ with $\mathbb E[\eta_u]=0$. 
Hence for any affine estimator $\widehat\psi_{t-j}$ based on $(\psi_t,\psi_{t+1})$, its bias is
\begin{equation}\label{eq:bias-det}
\mathrm{Bias}(\widehat\psi_{t-j})
= \mathbb E[\widehat\psi_{t-j}]-\psi^*(s_{t-j})
= \widehat\psi^{\,\mathrm{det}}_{t-j}-\psi^*(s_{t-j}),
\end{equation}
where $\widehat\psi^{\,\mathrm{det}}_{t-j}$ is obtained by replacing $(\psi_t,\psi_{t+1})$ with $(\psi^*(s_t),\psi^*(s_{t+1}))$.

Taylor expansion at $s_t$ up to order four gives
\begin{align}
\psi^*(s_{t+1})
&= \psi + \psi'\Delta + \tfrac12\psi''\Delta^2 + \tfrac16\psi'''\Delta^3 + \tfrac1{24}\psi''''\Delta^4 + O(\Delta^5),\label{eq:taylor-right}\\
\psi^*(s_{t-j})
&= \psi - j\psi'\Delta + \tfrac12 j^2 \psi''\Delta^2 - \tfrac16 j^3 \psi'''\Delta^3 + \tfrac1{24} j^4 \psi''''\Delta^4 + O(\Delta^5).\label{eq:taylor-left}
\end{align}
All $O(\cdot)$ terms are uniform in $t,j$ once bounded by
\[
M_r:=\sup_{s\in[0,1]}\|\psi^{*(r)}(s)\|,\qquad r=1,2,3,4.
\]

\textbf{Strategy (A).} 
By definition,
\[
\widehat\psi^{\,\mathrm{det},A}_{t-j}
=(j+1)\psi^*(s_t)-j\psi^*(s_{t+1})
=(j+1)\psi - j\bigl(\psi + \psi'\Delta + \tfrac12\psi''\Delta^2 + \tfrac16\psi'''\Delta^3 + \tfrac1{24}\psi''''\Delta^4\bigr)+O(\Delta^5).
\]
Simplifying gives
\begin{equation}\label{eq:detA}
\widehat\psi^{\,\mathrm{det},A}_{t-j}
= \psi - j\psi'\Delta - \tfrac{j}{2}\psi''\Delta^2 - \tfrac{j}{6}\psi'''\Delta^3 - \tfrac{j}{24}\psi''''\Delta^4 + O(\Delta^5).
\end{equation}
Subtracting \eqref{eq:taylor-left} from \eqref{eq:detA} yields
\[
\mathrm{Bias}(\widehat\psi^A_{t-j})
= -\tfrac12 j(j+1)\psi''\Delta^2
+ \tfrac16 (j^3-j)\psi'''\Delta^3
- \tfrac1{24} (j^4+j)\psi''''\Delta^4
+ O(\Delta^5),
\]
and therefore
\[
\|\mathrm{Bias}(\widehat\psi^A_{t-j})\|
\le \tfrac12 j(j+1)M_2\Delta^2 + O(j^3\Delta^3)=O(j^2\Delta^2).
\]

\textbf{Strategy (B).} 
For even $j=2k$, one has $\widehat\psi^{\,\mathrm{det},B}_{t-2k}=\psi$, hence
\[
\mathrm{Bias}(\widehat\psi^B_{t-2k})
= \psi - \psi^*(s_{t-2k})
= 2k\,\psi'\Delta - 2k^2\psi''\Delta^2 + \tfrac43 k^3\psi'''\Delta^3 + O(\Delta^4).
\]
This implies
\[
\|\mathrm{Bias}(\widehat\psi^B_{t-2k})\|\le 2kM_1\Delta + \tfrac12(2k)^2 M_2\Delta^2 + \tfrac16 (2k)^3 M_3\Delta^3 + O(\Delta^4)=O(j\Delta).
\]

For odd $j=2k+1$, one has $\widehat\psi^{\,\mathrm{det},B}_{t-(2k+1)}=2\psi-\psi^*(s_{t+1})$. 
Expanding and subtracting \eqref{eq:taylor-left} with $j=2k+1$ yields
\[
\mathrm{Bias}(\widehat\psi^B_{t-(2k+1)})
=2k\,\psi'\Delta - (2k^2+2k+1)\psi''\Delta^2 + O(\Delta^3),
\]
so that
\[
\|\mathrm{Bias}(\widehat\psi^B_{t-(2k+1)})\|\le 2kM_1\Delta + (2k^2+2k+1)M_2\Delta^2 + O(\Delta^3)=O(j\Delta).
\]

\textbf{Strategy (C).} 
Here $\widehat\psi^{\,\mathrm{det},C}_{t-j}=\psi$, so
\[
\mathrm{Bias}(\widehat\psi^C_{t-j})
= \psi - \psi^*(s_{t-j})
= j\psi'\Delta - \tfrac12 j^2 \psi''\Delta^2 + \tfrac16 j^3\psi'''\Delta^3 + O(\Delta^4),
\]
and thus
\[
\|\mathrm{Bias}(\widehat\psi^C_{t-j})\|\le jM_1\Delta + \tfrac12 j^2 M_2\Delta^2 + \tfrac16 j^3 M_3\Delta^3 + O(\Delta^4)=O(j\Delta).
\]

Combining the three strategies completes the derivation of the bias bounds and their leading expansions.
\end{proof}

\begin{theorem}[Variance growth]\label{thm:var}
Under the noise model above and independence across grid points, the variances are:
\[
\mathrm{Var}(\widehat\psi^{A}_{t-j})=\big((j{+}1)^2+j^2\big)\sigma^2 I_d,\qquad
\mathrm{Var}(\widehat\psi^{B}_{t-j})=
\begin{cases}
\sigma^2 I_d,& j\ \mathrm{even},\\
5\sigma^2 I_d,& j\ \mathrm{odd},
\end{cases}
\qquad
\mathrm{Var}(\widehat\psi^{C}_{t-j})=\sigma^2 I_d.
\]
\end{theorem}

\begin{proof}
Each estimator is a fixed linear combination $w_0\psi_t+w_1\psi_{t+1}$ with weights $(w_0,w_1)=((j{+}1),-j)$ for {\rm(A)}, $(1,0)$ or $(2,-1)$ for {\rm(B)}, and $(1,0)$ for {\rm(C)}. With homoscedastic uncorrelated noise, $\mathrm Var(w_0\psi_t+w_1\psi_{t+1})=(w_0^2+w_1^2)\sigma^2 I_d$.
\end{proof}

\begin{corollary}[Dominant orders of MSE]\label{cor:mse}
Let $h:=j\Delta$. As $\Delta\to0$ with $j$ possibly growing, the mean-squared errors satisfy
\[
\mathrm{MSE}(\widehat\psi^{A}_{t-j})=\Theta(h^4)+\Theta(j^2\sigma^2),\qquad
\mathrm{MSE}(\widehat\psi^{B}_{t-j})=\Theta(h^2)+\Theta(\sigma^2),\qquad
\mathrm{MSE}(\widehat\psi^{C}_{t-j})=\Theta(h^2)+\Theta(\sigma^2).
\]
\end{corollary}

\begin{proof}
Combine Theorems~\ref{thm:bias} and~\ref{thm:var}. The squared-bias for {\rm(A)} is $\Theta((j^2\Delta^2)^2)=\Theta(h^4)$, while for {\rm(B)} and {\rm(C)} it is $\Theta((j\Delta)^2)=\Theta(h^2)$. The variance orders are given in Theorem~\ref{thm:var}.
\end{proof}

\begin{remark}[Interpretation]
Strategy {\rm(A)} achieves second-order bias $\Theta(h^2)$ but pays a variance that grows quadratically with the number of skipped steps. Strategies {\rm(B)} and {\rm(C)} maintain constant variance independent of $j$ but can only guarantee first-order bias $\Theta(h)$. In low-noise, short-jump regimes, {\rm (A)} enjoys a strictly better MSE due to its $\Theta(h^4)$ squared bias. However, once the jump length exceeds one step or the noise level becomes higher, the variance term of {\rm (A)} dominates, and {\rm (B)}/{\rm (C)} may become preferable. In summary, extrapolation tends to overshoot, while reuse remains stable but sacrifices precision.
\end{remark}

%% file: texts/appendix_experiments.tex
\paragraph{Memory Complexity.} ZEUS is designed as a training-free plug-in with $\mathcal{O}(1)$ \emph{additional} memory: it keeps only a constant-size, output-level state (the observed information pair) and does not store layer-wise intermediate activations across timesteps. Therefore, ZEUS’s memory overhead is independent of the number of sampling steps and does not grow with more aggressive step skipping. We validate this claim empirically on FLUX-2-dev (64B) under 4-bit quantization by comparing two ZEUS settings with different speedups. As shown in Table~\ref{tab:flux2_mem}, peak GPU memory remains exactly the same (23{,}370\,MB) for ZEUS-Medium and ZEUS-Turbo, despite a substantial change in compute (speedup from $2.07\times$ to $3.34\times$). This ablation directly supports ZEUS’s $\mathcal{O}(1)$ additional-memory complexity.

\input{tables/rebuttal/memory}

%% file: tables/rebuttal/memory.tex

\begin{table}[t]
  \centering
  \scriptsize
  \caption{\textbf{Memory footprint ablation on FLUX-2-dev (64B) with 4-bit quantization.}
  ZEUS preserves the peak GPU memory budget while increasing speedup.}
  \label{tab:flux2_mem}
  \setlength{\tabcolsep}{4pt}
  \renewcommand{\arraystretch}{1.12}
  \begin{tabular}{@{}lccccc@{}}
    \toprule
\multicolumn{1}{c}{\textbf{Setting}} &
\multicolumn{1}{c}{\shortstack[c]{\textbf{Speedup}\\$\uparrow$}} &
\multicolumn{1}{c}{\shortstack[c]{\textbf{Peak Mem}\\\textbf{(MB)}~$\downarrow$}} &
\multicolumn{1}{c}{\shortstack[c]{\textbf{PSNR}\\$\uparrow$}} &
\multicolumn{1}{c}{\shortstack[c]{\textbf{SSIM}\\$\uparrow$}} &
\multicolumn{1}{c}{\shortstack[c]{\textbf{LPIPS}\\$\downarrow$}} \\
\midrule
    ZEUS (Medium) & $2.07\times$ & 23370 & 31.593 & 0.938 & 0.029 \\
    ZEUS (Turbo)  & $3.34\times$ & 23370 & 23.115 & 0.811 & 0.114 \\
    \bottomrule
  \end{tabular}
\end{table}